\documentclass[11pt]{article}
\usepackage[utf8]{inputenc}
\usepackage{amsmath, amssymb, amsthm}
\usepackage{geometry}
\usepackage{hyperref}
\usepackage{authblk}
\usepackage{xcolor}
\usepackage{mleftright}\mleftright
\definecolor{niceRed}{RGB}{190,38,38}
\definecolor{Red2}{RGB}{219, 50, 54}
\definecolor{mgreen}{RGB}{160, 200, 140}
\definecolor{blueGrotto}{RGB}{5,157,192}
\definecolor{limeGreen}{HTML}{81B622}
\definecolor{myellow}{rgb}{0.88,0.61,0.14}
\definecolor{darkGreen}{HTML}{2E8B57}
\definecolor{navyBlueP}{HTML}{03468F}
\definecolor{Sepia}{HTML}{7F462C}
\definecolor{red2}{HTML}{1F462C}
\definecolor{orange2}{HTML}{FF8000}
\definecolor{mgray}{HTML}{ABB3B8}
\definecolor{lgray}{HTML}{E5E8E9}
\definecolor{myPurple}{RGB}{175,0,124}
\definecolor{mypurple2}{rgb}{0.8,0.62,1}
\definecolor{royalBlue}{HTML}{057DCD}
\definecolor{mpink}{HTML}{FC6C85}
\definecolor{lblue}{RGB}{74,144,226}
\definecolor{peagreen}{RGB}{152,193,39}
\definecolor{typ_navy}{HTML}{001f3f}
\definecolor{typ_blue}{HTML}{0074d9}
\definecolor{typ_aqua}{HTML}{7fdbff}
\definecolor{typ_teal}{HTML}{39cccc}
\definecolor{typ_eastern}{HTML}{239dad}
\definecolor{typ_purple}{HTML}{b10dc9}
\definecolor{typ_fuchsia}{HTML}{f012be}
\definecolor{typ_maroon}{HTML}{85144b}
\definecolor{typ_red}{HTML}{ff4136}
\definecolor{typ_orange}{HTML}{ff851b}
\definecolor{typ_yellow}{HTML}{ffdc00}
\definecolor{typ_olive}{HTML}{3d9970}
\definecolor{typ_green}{HTML}{2ecc40}
\definecolor{typ_lime}{HTML}{01ff70}
\definecolor{newgreen}{HTML}{83c702}
\definecolor{newpurp}{RGB}{97,96,121}
\definecolor[named]{Purple}{cmyk}{0.55,1,0,0.15}
\definecolor[named]{DarkBlue}{cmyk}{1,0.58,0,0.21}
\hypersetup{colorlinks,
    colorlinks = true,
    linkcolor=DarkBlue,
    citecolor=Sepia,
    urlcolor=Purple,
    filecolor=DarkBlue
    linktocpage = true,
}
\usepackage{enumitem}
\usepackage{tikz}
\usepackage[capitalize,noabbrev]{cleveref}
\usepackage[style=alphabetic,natbib=true,maxbibnames=10]{biblatex}

\usetikzlibrary{arrows, arrows.meta,fit}
\pgfdeclarelayer{background}
\pgfsetlayers{background,main}

\usepackage{mathtools}
\newtheorem{theorem}{Theorem}
\newtheorem{lemma}{Lemma}
\newtheorem{proposition}{Proposition}
\newtheorem{corollary}{Corollary}

\newtheorem{definition}{Definition}

\newtheorem{remark}{Remark}
\newtheorem{fact}{Fact}

\newcommand{\R}{\mathbb{R}}
\newcommand{\E}{\mathbb{E}}

\DeclareMathOperator{\co}{conv}

\newcommand{\cF}{\mathcal{F}}
\newcommand{\cH}{\mathcal{H}}
\newcommand{\cX}{\mathcal{X}}
\newcommand{\cD}{\mathcal{D}}
\newcommand{\cL}{\mathcal{L}}

\newcommand{\cZ}{\mathcal{Z}}
\newcommand{\cO}{\mathcal{O}}
\newcommand{\cC}{\mathcal{C}}
\newcommand{\cY}{\mathcal{Y}}
\newcommand{\cP}{\mathcal{P}}
\newcommand{\ECE}{\mathrm{SwapECE}_T}
\newcommand{\yt}{\tilde{y}}

\newcommand{\cA}{\mathcal{A}}

\renewcommand{\epsilon}{\varepsilon}
\newcommand{\eps}{\epsilon}

\newcommand{\MCerr}{\textrm{\normalfont MC-Err}}
\newcommand{\Regret}{\textrm{\normalfont Regret}}
\newcommand{\EVI}{\textrm{\normalfont EVI}}

\newcommand{\poly}{\mathrm{poly}}

\DeclareMathOperator*{\argmin}{arg\,min}
\DeclareMathOperator*{\argmax}{arg\,max}

\usepackage[linesnumbered,ruled,lined,noend]{algorithm2e}
\input{algo_config}

\title{An Efficient Black-Box Reduction from Online Learning to Multicalibration, and a New Route to $\Phi$-Regret Minimization}
\date{}

\newcommand{\BR}{\sigma}

\author[1]{Gabriele Farina}
\author[1,2]{Juan Carlos Perdomo}
\affil[1]{MIT EECS}
\affil[2]{New York University}

\begin{document}

\maketitle

\begin{abstract}
We give a Gordon-Greenwald-Marks (GGM) style black-box reduction from online learning to online multicalibration. Concretely, we show that to achieve high-dimensional multicalibration with respect to a class of functions $\mathcal{H}$, it suffices to combine any no-regret learner over $\mathcal {H} $ with an expected variational inequality (EVI) solver. 
We also prove a converse statement showing that efficient multicalibration implies efficient EVI solving, highlighting how EVIs in multicalibration mirror the role of fixed points in the GGM result for $\Phi$-regret.

This first set of results addresses the high-dimensional analogue of the open question in Garg, Jung, Reingold, and Roth (SODA '24), showing that oracle-efficient online multicalibration with $\sqrt{T}$-type guarantees is possible in full generality. Furthermore, our GGM-style reduction unifies the analyses of existing algorithms, transfers guarantees from online learning to multicalibration for challenging environments with delayed observations or censored outcomes, and yields the first efficient black-box reduction between online learning and multiclass omniprediction.

Our second main result is a fine-grained reduction from high-dimensional online multicalibration to (contextual) $\Phi$-regret minimization. Together with our first result, this establishes a new route from external regret to $\Phi$-regret that bypasses sophisticated fixed-point or semi-separation machinery, dramatically simplifies a result of Daskalakis, Farina, Fishelson, Pipis, and Schneider (STOC '25) while improving rates, and yields new algorithms that are robust to richer deviation classes, such as those belonging to any reproducing kernel Hilbert space.
\end{abstract}
\clearpage
\tableofcontents
\clearpage

\vspace{1cm}

\section{Introduction}
Originally developed as a definition of algorithmic fairness, multicalibration has turned out to have interesting and important implications beyond its initial scope. 
Within learning theory, multicalibrated predictors are outcome indistinguishable; they constitute generative models that are computationally unfalsifiable \citep{dwork2021outcome}. 
Multicalibration also implies robustness to distribution shift \citep{kim2022universal} and loss minimization, not just for one loss, but for many losses simultaneously \citep{omnipredictors}. Within complexity theory, recent work has started to explore how multicalibration enables stronger versions of classical regularity lemmas, such as the Dense Model Theorem and the Hardcore Lemma \citep{silvia,lunjia}. In game theory, notions of calibration are intimately tied to swap regret minimization and thus to correlated equilibria \citep{foster1997calibrated,foster1998asymptotic,kakade2008deterministic}.

Despite this growing web of connections, the \emph{algorithmic} understanding of how to achieve multicalibration remains rather fragmented.  
In the batch setting, the only known algorithms are boosting-style procedures that use polynomially more samples than necessary. 
In the online setting, where procedures enjoy
worst-case guarantees and yield batch algorithms with optimal $\cO(1/\eps^2)$ sample
complexity via online-to-batch reductions, we have a complicated patchwork of
designs. 
Each algorithm targets a specific outcome type (e.g., binary) or hypothesis class (e.g., linear functions) and requires its own idiosyncratic and often lengthy analysis.
In short, we lack a clean algorithmic template that works in full generality.

There is, however, a compelling precedent for a solution to this issue. Multicalibration is reminiscent of a different and older problem in online learning, namely $\Phi$-regret minimization \citep{greenwald2003general}. Both are parameterized by a family of functions: tests in the case of multicalibration, and a set $\Phi$ of strategy transformations $\phi: \cZ \to \cZ$ in the case of $\Phi$-regret on a strategy set $\cZ$. Since the seminal result of \citet*{gordon2008no} (henceforth abbreviated as Gordon-Greenwald-Marks or GGM), it is known that $\Phi$-regret minimization can be efficiently reduced, black-box, to standard regret minimization over $\Phi$, converting the outputs over $\Phi$ into strategies in $\cZ$ by computing a \emph{fixed point} of the transformations. Since 2008, this black-box reduction has been the standard approach for $\Phi$-regret algorithms, and has led to a flurry of results for equilibrium computation (e.g., \citep{farina2022simple,fujii2023bayes,daskalakis2025efficient,arunachaleswaran2025swap} and references therein). 

In light of the above, this paper is motivated by two questions.
\begin{enumerate}[label={(Q\arabic*)}, left=1mm]
    \item\label{q1} Is there a Gordon-Greenwald-Marks-like reduction for online multicalibration? 
    That is, can online multicalibration be black-box reduced to regret minimization, and if so, what nonlinear optimization primitive plays the role that fixed points play for $\Phi$-regret?

    \item\label{q2} How formal is the resemblance between multicalibrated forecasting and $\Phi$-regret? Can one devise a more fine-grained connection between online $\cH$-multicalibration and $\Phi$-regret with $\cH$ scaling naturally with the complexity of $\Phi$ and recovering the well-known connection between strong calibration and swap regret in the limit?
\end{enumerate}

\noindent
We establish positive answers to both questions.

Our first contribution, addressing \ref{q1}, is a Gordon-Greenwald-Marks-style theorem for online multicalibration. This resolves the high-dimensional version of one of the open questions in \citet*{garg2024oracle} regarding the existence of efficient oracle-based constructions for online multicalibration. Specifically, we show that online multicalibration can be obtained efficiently and in a black-box fashion with respect to any class $\cH$ and set of high-dimensional outcomes $\cY$ by combining external regret minimization over test functions in $\cH$ with the nonlinear optimization primitive of expected variational inequalities (EVIs). This black-box reduction is especially appealing in light of recent algorithmic progress on EVIs, demonstrating that they can be solved very efficiently \citep{zhang2025expected}. Furthermore, we show that EVIs are not only sufficient but also necessary: contingent on an efficient oracle for online learning over $\cH$, an efficient online multicalibrated forecasting algorithm exists if and only if one can efficiently solve EVIs defined by the test functions in $\cH$. Importantly, our approach avoids binning or discretization arguments over the outcome space $\cY$ in order to preserve computational efficiency when outcomes $y$ are high-dimensional vectors.

Our reduction enables a technical transfer of ideas from the rich literature on online learning to the more recent area of multicalibration. Looking backward, we recover and dramatically simplify the proofs of prior online multicalibration algorithms. By leveraging recent advances in EVIs, these reinterpretations also yield stronger guarantees. Looking forward, we leverage our reduction to design new algorithms for previously unstudied versions of online multicalibration. In particular, we develop forecasting procedures with guarantees in challenging learning environments, such as settings with delayed feedback or censored observations, simply by swapping in the appropriate off-the-shelf regret minimizer.
Given how online problems generate offline problems, these new algorithms also imply better batch algorithms via online-to-batch reductions.

Our second contribution, addressing \ref{q2}, runs in the opposite direction: from forecasting to decision making. We show that an efficient no-$\Phi$-regret algorithm can be constructed by first building an online forecaster for the losses that is multicalibrated with respect to a set of tests defined in terms of $\Phi$,
and then best responding. This result yields a substantially finer-grained
reduction from decision making to forecasting than the folklore connection between swap regret
and $\ell_1$-multicalibration (e.g., \citep{foster1997calibrated}), and yields state-of-the-art $\Phi$-regret algorithms for new domains.

Taken together, we show that our two main reductions provide an alternative path between external regret and $\Phi$-regret to that established by \citet{gordon2008no}: Rather than optimizing directly over endomorphisms and then computing fixed points, one can route through multicalibration and EVIs.
Recent work on $\Phi$-regret minimization for expressive deviation classes has increasingly run into the limits of the GGM path: extending it to richer classes of deviations over general action sets $\cZ$ has required increasingly elaborate machinery, including semi-separation oracles and careful relaxations of sets of endomorphisms \citep{daskalakis2025efficient}. Certain natural classes, such as deviations lying in any reproducing kernel Hilbert space (RKHS), have remained out of reach. Our multicalibration route provides a simplified technical approach that not only recovers prior results on linear and low-degree polynomial swap regret \citep{daskalakis2025efficient,arunachaleswaran2025swap,zhang2025learning} but also establishes the first known no-$\Phi$-regret algorithms for new settings, such as RKHSs.

\subsection{Background: Online Forecasting and Decision-Making}

The focus of our work is on understanding the algorithmic landscape and connections between online multicalibrated forecasting algorithms and online decision-making algorithms for $\Phi$-regret. We start by defining both of these concepts formally. 

While much of the literature has focused on one-dimensional, \emph{binned} notions of multicalibration, in this paper we are interested in the high-dimensional, non-binned definition that has recently been studied (\emph{e.g.}, \citet[Def. 1]{deng2023happymap} and \citet[Def. 3.2]{dwork2023pseudorandomness}).
As we show in \Cref{sec:ece}, this notion generalizes prior popular instantiations in the literature, such as $\ell_1$ notions of calibration (also known as expected calibration error or ECE), where errors are summed over a grid of discretized bins. See \Cref{rem:high dim} in \Cref{sec:ece} for further discussion. 

\begin{definition}[Online Multicalibration]\label{eq:omc}
Consider the online protocol where at every time  $t$, Nature selects features $x_t \in \cX$, and then the Learner produces a distribution $\cD_t$ over forecasts $p_t \in \cY \subset \R^d$ where $\cY$ is a compact, convex set. Lastly, Nature reveals the outcome $y_t \in \cY$.

An algorithm $\cA$ guarantees online multicalibration with respect to a class $\cH \subseteq \{\cX \times \cY \to \R^d\}$, if regardless of Nature's choices of $(x_t,y_t)$, the multicalibration error
\begin{align*}
\MCerr_T(h) \coloneqq \sum_{t=1}^T \E_{p_t \sim \cD_t} \Big[ h(x_t,p_t)^\top (y_t - p_t) \Big] 
\end{align*}
grows sublinearly as a function of $T$ for all test functions $h \in \cH$.
For brevity, we adopt the shorthand $\MCerr_T \coloneqq \sup_{h\in \cH} \MCerr_T(h)$.
\end{definition}

Intuitively, the sequence of distributions $\cD_t$ is multicalibrated if the expected errors, $p_t -y_t$, are uncorrelated with any function in the class $\cH$. Here, we stated the definition without absolute values, since it allows for finer-grained results. However, one can always recover the absolute values by requiring $\cH$ to be symmetric, meaning that for any $h\in\cH$, $-h$ is also in the class.

Next, we define $\Phi$-regret. We opt to present the slightly more permissive definition described by \citep{zhang2024efficient}, which enables the Learner to output \emph{distributions} over points in $\cZ$ at every round. This definition preserves connections between $\Phi$-regret and $\Phi$-equilibria, but enables efficient algorithmic paths that the original definition of \citet{greenwald2003general} precluded when $\Phi$ contains nonlinear functions.\footnote{Conversely, when all transformations in $\Phi$ are linear, the expectation can be pushed through the deviation, and any $\Phi$-regret minimizer in the more permissive sense of \cref{def:phi-regret} can be coalesced into the less permissive definition in which a single point (the expectation of $\mu_t$) is output.}

\begin{definition}[$\Phi$-Regret Minimization]\label{def:phi-regret}
Fix a set of deviations $\Phi \subseteq \{\cZ \to \cZ\}$ where $\cZ \subset \R^d$ is a compact, convex set of actions. Consider the online protocol in which, at each time $t$, the Learner selects a distribution $\mu_t$ over actions in $\cZ$, and then Nature reveals a loss vector $\ell_t \in \R^d$ from some set $\cL$. An algorithm guarantees no-$\Phi$-regret if it selects $\mu_t$ such that no matter Nature's choices of $\ell_t$
\[
\Phi\textup{-Regret}_T(\phi) \coloneqq \sum_{t=1}^{T} \E_{z_t \sim \mu_t}\Big[\ell_t^\top z_t - \ell_t^\top \phi(z_t) \Big]
\]
grows sublinearly as a function of $T$ for all transformations $\phi \in \Phi$. We adopt the shorthand $\Regret_T \coloneqq \sup_{\phi\in\Phi} \Regret_T(\phi)$.
\end{definition}
When $\Phi$ is the set of all \emph{constant} functions, $\Phi$-regret reduces to standard (also known as \emph{external}) regret.
At the opposite extreme, when $\Phi$ is the set of all functions from $\cZ$ to $ \cZ$, it captures full \emph{swap} regret. Swap regret is, however, known to be hard to minimize over general convex sets $\cZ$ \citep{daskalakis2024lower}, motivating a line of work proposing algorithms and definitions that guarantee low regret with respect to increasingly rich, yet tractable classes of deviations (e.g., \citep{morrill2021efficient,farina2023polynomial,arunachaleswaran2025swap,farina2022simple,mansour2022strategizing,fujii2023bayes} and references therein).

\paragraph{The GGM paradigm.}
The celebrated result of \citet{gordon2008no} established that $\Phi$-regret minimization can be reduced, in a black-box fashion, to external regret minimization over the class $\Phi$. The nonlinear optimization primitive enabling this reduction is the computation of \emph{expected fixed points} (EFPs),\footnote{Technically, \citet{gordon2008no} uses fixed points, but their result goes through directly for our more permissive definition of $\Phi$-regret if one instead uses \emph{expected} fixed points, which are easier to compute. See \citet{zhang2024efficient} for further discussion.} which solve the following problem. 

Given an endomorphism $\phi: \cZ \to \cZ$ and $\eps > 0$, find a polynomially-supported discrete distribution $\mu$ over $\cZ$  such that 
\begin{align*}
   \big\| \E_{x\sim \mu}[\phi(x) - x] \big\|_2 \leq \eps.
\end{align*}
With that primitive, the following result can be shown.
\begin{theorem}[\cite{gordon2008no}, Informal]
Let $\cA$ be any external regret minimizer for the set of transformations $\Phi$. Then, given the ability to efficiently solve expected fixed points for any $\phi \in \Phi$, $\cA$ can be efficiently converted in a black-box fashion into an efficient online algorithm $\cA'$ that guarantees no-$\Phi$-regret over $\cZ$.
\end{theorem}

This reduction has been the \emph{de facto} algorithmic template for $\Phi$-regret minimization and its related equilibrium notions. As an example, the classical reduction from swap regret to external regret of \citet{blum2007external}, and from internal regret to external regret of \citet{stoltz2005internal}, are particular instantiations of this framework.\footnote{In \citet{blum2007external}, the expected fixed point computation subproblem is solved by computing the stationary distribution of a Markov chain.} Furthermore, as pointed out by \citet{hazan2007computational}, the use of expected fixed points is not accidental. Any efficient algorithm for no-$\Phi$-regret implies an efficient algorithm for solving expected fixed points for any $\phi$.

\subsection{Contributions and Techniques}
\begin{figure}[t]
    \centering%
    \begin{tikzpicture}
        \tikzstyle{box}=[draw, text width=3.5cm, align=center, fill=white, fill opacity=0.9,inner ysep=2mm, rounded corners=.2mm, semithick];
        \tikzstyle{group}=[fill, semithick, draw, fill opacity=0.08, dotted, inner sep=1.65mm, rounded corners=2mm]
        \tikzstyle{ar}=[draw, semithick, shorten <=.5mm, shorten >=.5mm, -{Stealth[scale=1.3]}];
        \node[box] (R1) at (-5,0) {External regret minimizer for $\cH$};
        \node[box] (R2) at (5,0) {External regret minimizer for $\Phi$};
        \node[box] (MC) at (-5,-3) {$\cH$-multicalibrated forecaster for $\cY$};
    
        \node[box] (GEN) at (-5,-5) {Outcome Indist., Omniprediction, $\dots$};
        \draw[ar,{Stealth[scale=1.3]}-{Stealth[scale=1.3]}] (MC) --node[right] {} (GEN);

        \node[box] (PHI) at (5,-3) {$\Phi$-regret minimizer for $\cX$};
    
        \draw[ar] (R1) --node[pos=.45,right,text width=2cm] {EVI} (MC);
        \draw[ar] (MC) --node[above] {Best Response} node[below] {} (PHI);
        \draw[ar] (R2) --node[pos=.45,right, text width=2cm] {Expected\\ fixed point} (PHI);


        \begin{pgfonlayer}{background}
            \node[group, dotted, blue, fill=blue!30!white, fit=(R2) (PHI)] (GGM) {};
            \node[blue, above, text width=5.3cm, align=center] at (GGM.north) {\citet{gordon2008no}};
        
            \node[group, red, solid, fit=(R1) (MC)] (NEW1) {};
            \node[red,above, text width=5.2cm, align=center] at (NEW1.north) {\Cref{thm:ol to mc}~(\Cref{sec:ol to mc})};
            
            \node[group, violet, solid, fill opacity=0.06, inner sep=3mm, fit=(MC) (PHI)] (NEW2) {};
            \node[violet,above, text width=5.2cm, align=center] at (NEW2.north) {\Cref{thm:mc to phi}~(\Cref{sec:mc to phi})};
        \end{pgfonlayer}
    \end{tikzpicture}
    \caption{Conceptual overview of results. Our first result reduces online multicalibration to online learning via EVIs. The second reduces $\Phi$-regret minimization to multicalibration using best responses. Together, these open a new path from regret minimization to $\Phi$-regret based on forecasting that has various computational advantages.}
    \label{fig:reductions}
\end{figure}
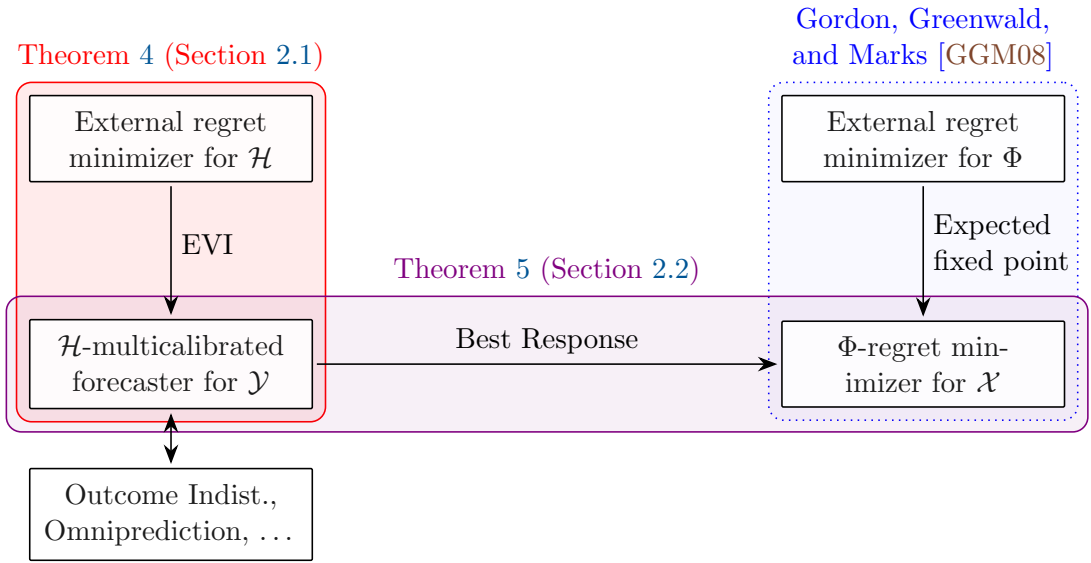

\paragraph{A GGM-Style Reduction for Online Multicalibration.} Our first result establishes an analogous result to that of \cite{gordon2008no}, but this time between online learning and multicalibration, rather than $\Phi$-regret. We show that multicalibration with respect to a set $\cH$ can be efficiently reduced, in a black-box fashion, to online learning over the class $\cH$ using a nonlinear optimization tool called an \emph{expected variational inequality} (EVI).

EVIs have appeared in different literatures with different names, including
``outgoing minimax problems'' \citep{foster2021forecast}, ``accuracy certificates'' \citep{nemirovski2010accuracy}, or ``negative correlation search'' \cite{perdomo2025defense}.
Given a function $S:\cY \to \R^d$ from a set $\cY\subset \R^d$ to $\R^d$ and error $\eps > 0$, the solution to the corresponding EVI is a distribution $\cD$ over $\cY$ such that 
\begin{equation}\tag{EVI}\label{eq:evi body}
    \E_{p \sim \cD} \Big[S(p)^\top (y-p)\Big] \le \eps \qquad \forall y \in \cY.
\end{equation}
As we recall in \cref{fact:EVI} below, EVIs can be solved efficiently, with runtime scaling as $\log(1/\eps)$ as a function of the desired precision $\eps$ for any convex and compact set $\cY$ and general (nonmonotone, or even discontinuous) operator $S$, under minimal assumptions. 

The following theorem summarizes our first main result.

\begin{theorem}[Informal]
The following statements are true for a general high-dimensional outcome space $\cY \subseteq \R^d$ and function class $\cH \subseteq \{\cX \times \cY \to \R^d\}$.
\begin{itemize}
    \item (Online Learning and EVI $\to$ Multicalibration) Let $\mathcal{A}$ be any online optimization algorithm over the class $\cH$ with external regret bounded by $\Regret_T$. 
    When paired with an EVI solver, we obtain an online algorithm whose multicalibration error with respect to $\cH$ is at most $\Regret_T$.
    
    \item (Multicalibration $\to$ EVI) Let $\cA$ be any algorithm that efficiently guarantees online multicalibration with respect to $\cH$, then $\cA$ can be black-box converted into an efficient algorithm for solving EVIs over $\cH$.
\end{itemize}
\end{theorem}

The reduction and its proof fit in a few lines (see \Cref{thm:ol to mc}), and are fully black-box. It makes no assumptions on $\cH$, works for any convex compact $\cY$, does not rely on binning arguments, and is hence amenable to general high-dimensional problems. As an immediate consequence, it resolves the high-dimensional version of one of the open questions in \citet*{garg2024oracle}, who asked whether oracle-efficient multicalibration algorithms with $\sqrt{T}$ rates exist. 

As with the celebrated GGM theorem for $\Phi$-regret, this general reduction has several implications for the algorithmic landscape of online multicalibration. To start, we show how previous algorithms with lengthy, elaborate proofs can be recovered and improved by selecting different online learning algorithms and EVI solvers in the reduction. These include the K29 algorithm, as proposed in \citet*{vovk2005defensive} and extended in \citet{kernelOI,farina2026defensive}, and the unbiased prediction algorithm of \citet*{noarov2023high}. Using the online learning algorithm from \cite{chen2021impossible} as the online learning, a special case of our reduction yields an algorithm which achieves $\ell_1$ or ECE calibration with respect to vector-valued outcomes at the optimal $\cO(T^{\frac{d+1}{d+2}})$ rate (\Cref{sec:ece}).

Beyond its conceptual simplicity and ability to recover known results, we also show how to use our reduction to derive simple algorithms for more complex multicalibrated forecasting tasks. 
In particular, by drawing on the well-developed literature on online learning, we design the first online multicalibration algorithms that succeed even in the face of delayed or censored outcomes. All of these follow as simple corollaries of \Cref{thm:ol to mc}. 

Furthermore, our contributions for online multicalibration flow through to yield improved algorithms for solution concepts that are implications of multicalibration, such as online omniprediction or outcome indistinguishability. In particular, our results yield an efficient black-box reduction for online multiclass omniprediction for general loss functions and hypothesis classes.

\paragraph{An Alternative Path to $\Phi$-regret.}
Our second main contribution runs in the opposite direction, from forecasting to decision-making.
\begin{theorem}[Informal]
Let $\sigma(p) \in \arg\min_{z \in \cZ} \; p^\top z$ denote the best-response map. For every deviation $\phi \in \Phi$, define the test function $h_\phi(p) \coloneqq \sigma(p) - \phi(\sigma(p))$, and let $\cH_\Phi \coloneqq \{h_\phi : \phi \in \Phi\}$. Then, any online forecaster of the losses $\ell_t$ that is multicalibrated with respect to the class $\cH_{\Phi}$ can be black-box converted to an efficient no-$\Phi$-regret algorithm by simply best-responding to the forecasts.
\end{theorem}
 
This theorem yields a substantially finer-grained reduction from decision-making to forecasting than the folklore connection between the demanding notion of swap regret and the demanding notion of $\ell_1$-calibration. As discussed previously, swap regret minimization is inefficient on general sets (e.g., \citet{daskalakis2024lower}), which begs the question: What is the richest class of deviations $\Phi$ that one can handle?
We address that question by establishing a more precise connection between deviations $\phi$ and multicalibration tests $h_\phi$, thus obtaining reductions in which the required multicalibration class scales naturally with the complexity of $\Phi$. This fine-grained reduction recovers the classical connection between strong calibration and swap regret only in the limit when $\Phi$ is the set of all deviations.

Together with the previous result, this theorem gives an alternative to the reduction of \citet{gordon2008no} (see \Cref{fig:reductions}). This route, grounded in forecasting and EVIs rather than fixed points, has features that make it appealing when the set of deviations $\Phi$ is large. Perhaps most importantly, unlike the GGM framework, this route does not require an understanding of the geometry of endomorphisms of the action space, dispensing with much of the technical complexity in recent extensions. We highlight that this leads to a substantially different construction for the recent STOC paper on linear swap regret \citep{daskalakis2025efficient}: we improve the rates, completely avoid the semiseparation machinery needed there to handle endomorphisms, and even obtain an extension to general RKHS deviation classes, yielding the first algorithm for that setting. Furthermore, we highlight that these results hold not just for $\Phi$-regret, but also for \emph{contextual} versions of $\Phi$-regret as we explain in \Cref{sec:reductions}.

We close this section by noting that online linear optimization is a special case of $\Phi$-regret minimization. Online linear optimization and $\Phi$-regret minimization are equivalent when the set $\Phi$ only contains a convex set of constant transformations. Therefore, our results show that given a general algorithm for online multicalibration, one can build a general algorithm for online linear optimization, and \textit{vice versa}, highlighting a computational equivalence between the two problems.

\subsection{Related Work}
\label{sec:related_work}

Our work lies at the intersection of various areas. Here, we briefly outline how our results build on and contribute to the most closely related work in each.

\paragraph{Multicalibration.} Multicalibration was introduced in the batch setting by \citet*{hebert2018multicalibration} as a definition of algorithmic fairness. As discussed previously, multicalibration has a wide range of interesting consequences, including outcome indistinguishability \citep{dwork2021outcome}, omniprediction \citep{omnipredictors,gopalan2022loss,kernelOI,okoroafor2025near}, and universal adaptability \citep{kim2022universal}. Within complexity theory, it implies stronger versions of the Dense Model theorem, the Hardcore Lemma, and other pseudorandomness notions \citep{dwork2023pseudorandomness,casacuberta2025global,silvia,lunjia}.

Batch notions of multicalibration are achievable via boosting algorithms that repeatedly call a weak agnostic learning oracle. These procedures are sample inefficient by at least a quadratic factor \citep{hebert2018multicalibration}. In the online setting, where $\sqrt{T}$ regret algorithms imply sample-efficient batch procedures, algorithms date back at least to \cite{vovk2005defensive} and \cite{foster2006calibration}, with increased activity following the work by \citet*{gupta2021online}. In the setting where outcomes are high-dimensional, algorithms are known for special cases such as when $\cH$ is any finite set of functions \citep{noarov2023high}, or any vector-valued reproducing kernel Hilbert space \citep{farina2026defensive}.

Prior work has established reductions between multicalibration and regret minimization in specific settings. \cite{haghtalab2023unifying} show a reduction for the restricted case of multiclass outcomes. However, their result is both white-box and exponential-time since it relies on constructing an $\eps$-net of the $d$-dimensional simplex. \cite{garg2024oracle} prove a reduction between swap multicalibration and online regression that works for the specific case of binary outcomes and yields $\omega(\sqrt{T})$ rates. They leave as their main open question how to achieve an efficient black-box reduction for (vanilla) online multicalibration that achieves $\sqrt{T}$ rates, a problem we solve not just for binary $y$ but for general high-dimensional outcomes. \cite{singla} show a similar reduction for the case of binary outcomes and $\ell_1$ multicalibration, a distinct notion from ours for which $\sqrt{T}$ rates are unachievable \cite{collina2026optimal}. Most recently, \citet{hu2025efficient} showed an oracle-efficient reduction from $\ell_p$ calibration to weak agnostic learning solving the open question in \cite{garg2024oracle} for the special case of binary outcomes. Their construction differs from ours in at least two important directions. First, their results rely on discretization and binning arguments. They are limited to binary outcomes and do not apply to general high-dimensional or real-valued outcomes $\cY$. Second, their reduction uses multiple learners which are trained conditionally, unlike our reduction which only uses one (\cref{algo:mc from ol}).

\paragraph{$\Phi$-Regret Minimization.} Virtually all algorithms for $\Phi$-regret minimization rely, either directly or indirectly, on the GGM result that uses fixed points. This is true of the celebrated no-swap-regret algorithm for normal-form games of \citet{blum2007external} and the no-internal-regret for normal-form games of \citet{stoltz2005internal}. Both are specific instantiations of the general reduction in \citet{gordon2008no}. GGM also applies to several important algorithms for extensive-form games, including constructions of efficient algorithms for extensive-form correlated, extensive-form coarse-correlated equilibria, and linear correlated equilibria \citep{morrill2021efficient,farina2022simple,farina2024polynomial}. 
In contemporary work, the GGM framework has been pushed to its limits. Recent work has resorted to highly sophisticated semi-separation arguments to achieve regret minimization via the GGM path with respect to ever-increasing sets of transformations $\Phi$ and general convex action sets \citep{daskalakis2025efficient,zhang2025learning,arunachaleswaran2025swap}. Our alternative path is simpler and achieves rates that match or improve upon those discussed here.

\paragraph{EVIs and Convex Sets.} Lastly, our reductions heavily rely on the ability to solve expected variational inequalities, as defined in \eqref{eq:evi body}. Recent work has developed efficient polynomial-time algorithms for solving these for any bounded operator $S$ on convex sets $\cZ$ given an oracle (an algorithm that tests membership, performs separation, or linear optimization over $\cZ$, all of which are known to be polynomial-time-reducible to each other under minimal regularity conditions \citep{Grotschel1993,grotschel1981ellipsoid}) for $\cZ$. In particular, the only assumption that we make is that $\cZ$ has been shifted and scaled appropriately so that the following condition holds.

\begin{definition}[Well-bounded set]\label{def:wr}
We say that a convex and compact set $\cZ \subseteq \R^d$ is well-bounded if it contains in its relative interior a unit-radius Euclidean ball, and $\cZ$ is contained in a Euclidean ball of radius $\poly (d)$ centered in the origin. \end{definition}

For any convex body $\cZ$ for which radii $r, R > 0$ are known such that a ball of radius $r$ is contained in the interior of $\cZ$ and $\cZ$ is contained in the ball of radius $R$ centered in the origin, a transformation that makes $\cZ$ well-bounded can be computed in time polynomial in $\log (R/r)$ using a polynomial number of oracle calls \citep[Corollary 4.6.8]{Grotschel1993}. The assumption of well-boundedness allows us to ignore the terms $\log(R/r)$, and simply replace them with $\poly(d)$. Under well-boundedness assumptions, we recall the following result.

\begin{fact}[\cite{zhang2025expected}]
\label{fact:EVI}
Given a well-bounded convex compact set $\cZ \subseteq \R^d$ and an operator $S:\cZ \to \R^d$ such that $\sup_{z\in \cZ} \|S(z)\|\leq B$, there exists an efficient algorithm that solves the EVI problem over $S$ to error $\eps$ in  $\cO(\poly(d,\log(B/\eps),\mathrm{eval}(S)))$ time, where $\mathrm{eval}(S)$ is an upper bound on the time it takes to evaluate $S$. For any $\eps> 0$, the algorithm returns a distribution $\cD$ such that $\forall y \in \cZ$,
\begin{align*}
    \E_{p \sim \cD} \Big[S(p)^\top (y-p)\Big] \le \eps,
\end{align*}
where $\cD$ has positive mass on at most  $\cO(\poly(d, \log(B/\eps)))$ points.
\end{fact}

This fact, shown in \cite{zhang2025expected}, uses a generalization of the ellipsoid-against-hope algorithm \cite{farina2024polynomial,papadimitriou2008computing, jiang2011polynomial}. However, slower algorithms that avoid the ellipsoid method also exist. For example, one can always solve EVIs via online linear optimization over $\cZ$. This regret minimization procedure solves EVIs to error $\eps$ in  $\cO(\poly(d,B,\mathrm{time}(S))/\eps^2))$ time and returns a distribution with support size $\cO(\poly(d,B))$.
Both algorithmic paths (the fast one of \cref{fact:EVI} and the slower approach based on regret minimization) are essentially enabled by algorithmic constructions of the minimax theorem (see also \citep{sigecom} for a high-level discussion).

\paragraph{Blackwell Approachability.} Dating back (at least) to the work of \citet{foster1999proof}, tools from Blackwell approachability \cite{blackwell1956analog} have been repeatedly used as subroutines to build online algorithms for calibration and multicalibration via tailored, problem-specific constructions.
For the restricted case of 1-dimensional outcomes, \citet{okoroafor2024faster,okoroafor2025near, perdomo2025defense}
encode finitely many discretized (multi)calibration conditions as the coordinates
of a vector-valued game and employ approachability to achieve online calibration. Going beyond finite constraints,
\citet{perchet2013approachability} defines a measure over calibration constraints and uses infinite-dimensional
approachability for weighted average of these constraints. These results, however, rely on binning and discretization arguments, and are hence not computationally efficient in high-dimensional settings. Furthermore, the Blackwell game usually has one payoff for each test, putting into question this route for richer test sets such as the RKHS tests we discuss in \Cref{sec:defensive} and use in \Cref{sec:applications Phi}. In contrast, our reductions are computationally efficient even for vector-valued outcomes, and can guarantee multicalibration with respect to \textit{every} function in an infinitely large class of tests (see the example in \Cref{sec:ol to mc}). 

On a more conceptual level, our work establishes black-box reductions between online linear optimization (OLO is a special case of $\Phi$-regret) and online high-dimensional multicalibration, highlighting a new computational equivalence between these problems. This result complements prior work by \citet{abernethy2011blackwell} who established an equivalence between Blackwell approachability and online linear optimization. Tying both sets of results together, our EVI-based reductions highlight how Blackwell approachability is not just a bespoke tool for the design of specific calibration algorithms, but rather computationally equivalent to the most general version of this class of problems (that is, online high-dimensional multi-calibration).

\section{A GGM-Like Theorem for Multicalibration, and Finer-Grained Connections between Decision-Making and Forecasting}
\label{sec:reductions}

In this section, we state our first result showing that multicalibration can be black-box reduced to online learning, and that $\Phi$-regret minimization can be black-box reduced to multicalibration. These reductions enable a number of algorithmic corollaries, which we discuss in \Cref{sec:mc_implications,sec:applications Phi}. 

\subsection{Reducing Multicalibration to Online Learning via EVIs}
\label{sec:ol to mc}
\begin{algorithm}[t]
    \caption{Black-box reduction from  online learning over $\cH$ to $\cH$-multicalibration}
    \label{algo:mc from ol}
    \DontPrintSemicolon
    \KwData{A no-regret learner $\cA$ over the set $\cH \subseteq \{ \cX \times \cY \to \R^d \}$}
    \For{$t=1,2,\dots$}{
    \hspace{10pt}Nature reveals $x_t$\;
    \hspace{10pt}Regret minimizer $\cA$ for the set $\cH$ picks next $h_t$\;
    \hspace{10pt}Learner outputs the distribution $\cD_t$ over predictions $p_t \in \cY$ that solves the EVI,
    \begin{equation}\label{eq:evi}
        \E_{p\sim\cD_t} \Big[h_t(x_t, p)^\top (y - p)\Big] \le \eps_t \quad\forall y \in \cY.
    \end{equation}\;
    \vspace{-5mm}
    \hspace{10pt}Nature reveals outcome $y_t \in \cY$\;
    \hspace{10pt}Regret minimizer incurs loss $f_t(h)= -\E_{p_t\sim\cD_t}\Big[ h(x_t,p_t)^\top (y_t - p_t) \Big]$}
\end{algorithm}

Let $\cX$ be a set of features, $\cY \subseteq \R^d$ be a compact, convex set of outcomes, and $\cH \subseteq \{\cX \times \cY \to \R^d \}$ be a set of bounded test functions. Let $\cA$ be any no-regret learner for $\cH$, that is, an algorithm that selects $h_t$, observes linear losses $f_t$, and guarantees sublinear external regret
\begin{align}
\label{eq:ext_regret}
     \Regret_T(h) \coloneqq \sum_{t=1}^T f_t(h_t) -  f_t(h)
\end{align}
with respect to any $h \in \cH$. We again let $\Regret_T = \sup_{h \in\cH} \Regret_T(h)$.

\Cref{algo:mc from ol} shows how $\cA$ can be black-box converted into an online multicalibration procedure over $\cH$. The idea is simple. We use $\cA$ to determine which test function in $\cH$ to prioritize in each round. Given the current context $x_t$, and the test $h_t$ selected by $\cA$, we convert $h_t$ into a distribution $\cD_t$ over the forecasts $p_t \sim \cD_t$ by solving the EVI corresponding to $h_t(x_t, \cdot)$, as shown in \eqref{eq:evi}; we recall that EVIs can be solved efficiently under minimal assumptions (see \cref{fact:EVI}). Finally, upon receiving the actual outcome $y_t \in \cY$ from the environment, we supply the loss $f_t(h)$ (which is always linear in $h$) back to the no-regret algorithm $\cA$.

The following theorem is our first main result establishing the correctness of this reduction.

\begin{theorem}\label{thm:ol to mc}

    Fix $h \in \cH$ and let $\Regret_T(h)$ denote the regret cumulated by the external regret minimizer $\cA$  (as in \Cref{eq:ext_regret}) 
    where $f_t$ are defined as in \Cref{algo:mc from ol} and $\mathrm{EVI}_T \coloneqq \sum_{t=1}^T \epsilon_t$ is the sum of the EVI errors from \Cref{eq:evi}, then the sequence of distributions $\cD_t$ output by the Learner in \Cref{algo:mc from ol} satisfy the following inequality,
    \[
        \MCerr_T(h) \le \Regret_T(h) + \mathrm{EVI}_T.
    \]
\end{theorem}
\begin{proof}
The key is that the EVI property guarantees that
\[
    0 \le \eps_t- \E_{p_t\sim\cD_t} \Big[h_t(x_t, p_t)^\top (y_t - p_t)\Big].
\]
Therefore,
\begin{align*}
\sum_{t=1}^T \E_{p_t\sim\cD_t} \Big[ h(x_t,p_t)^\top (y_t - p_t) \Big]    & \leq  \sum_{t=1}^T \E_{p_t\sim\cD_t} \Big[ h(x_t,p_t)^\top (y_t - p_t)\Big] - \E_{p_t\sim\cD_t} \Big[h_t(x_t, p_t)^\top (y_t - p_t)\Big]  + \eps_t\\
    &= \left(\sum_{t=1}^T f_t(h_t) - f_t(h)\right) + \sum_{t=1}^T \eps_t \\
    &=  \Regret_T(h) + \mathrm{EVI}_T.\qedhere
\end{align*}
\end{proof}
Since $\EVI_T$ can always be chosen to be less than $\Regret_T$, we get that if $\cA$ guarantees that $\Regret_T$ is $o(T)$, then the multicalibration error with respect to $\cH$ is also $o(T)$.

Before we discuss the necessity of EVIs, we briefly pause to illustrate how \cref{thm:ol to mc} can be used to easily design online multicalibration algorithms in two important settings. We keep the discussion short and informal to focus on the overarching idea and defer a detailed presentation of (stronger versions of) these results to \Cref{sec:mc_implications}.

\paragraph{Multicalibration with respect to finite sets.}
To start, suppose we would like to produce forecasts of high-dimensional outcomes $y \in \cY \subset \R^d$ that are multicalibrated with respect to all functions in the convex hull of a \emph{finite} set of tests $\cH$. That is, $\cH = \co\{h^{(1)}, \dots, h^{(n)}\}$ for some $n \in \mathbb{N}$. We assume that evaluating each $h^{(i)}$ incurs a cost polynomial in the dimension $d$, and that $|h^{(i)}(x,p)^\top (y - p)| \le 1$ for all $i \in [n]$, $x \in \cX$, and $y,p\in\cY$.

Applying \cref{algo:mc from ol} using Hedge as the no-regret algorithm over $\cH$ yields a  forecaster for $\cY$ that is multicalibrated with respect to each $h\in \cH$, that runs in time $\cO(\poly(nd)\log t)$ at every iteration $t$, and achieves multicalibration error 
\begin{align}
\label{eq:hedge_example}
\MCerr_T(h) \leq \cO(\sqrt{T \log n}).
\end{align}
In more detail, at every time $t$, Hedge outputs a function $h_t$ which is a convex combination of the functions in $\cH$, $h_t = \lambda^{(1)}_{t} h^{(1)} + \dots + \lambda^{(n)}_{t} h^{(n)} \in \cH$. Evaluating such $h_t$ incurs a cost $\cO(\poly(nd)).$ Using the EVI algorithm of \cite{zhang2025expected}, we can then solve each EVI induced by $h_t(x_t, \cdot)$ to precision $1/t$ in time $\cO(\poly(nd)\log t )$, yielding the distribution $\cD_t$ with support size $\cO(\poly(d)\log t)$ and $\EVI_T = \sum_{t=1}^T 1/t \leq \cO(\log(T))$.

Upon observing the realized $y_t \in \cY$, we can then update the convex combination coefficients $\lambda_{t,1}, \dots, \lambda_{t,n}$ according to the Hedge algorithm. In particular, in light of the definition of $f_t$ in \Cref{algo:mc from ol}, each coefficient is updated according to
\[
    \lambda^{(i)}_{t+1} \propto \lambda^{(i)}_{t} \cdot \exp\bigg\{ \eta \E_{p_t \sim \cD_t} \Big[ h^{(i)}(x_t, p_t)^\top (y_t - p_t) \Big] \bigg\}
\]
for some appropriate learning rate $\eta > 0$.
Since $\cD_t$ has support $\cO(\poly(d)\log t)$, the update can be carried out in time $\cO(\poly(nd)\log t)$. 
Invoking \cref{thm:ol to mc} with the known regret bound of Hedge, we obtain the bound in \Cref{eq:hedge_example}.

\paragraph{Multicalibration with respect to infinitely many tests.}

The reduction we devised is useful well beyond (convex hulls of) discrete tests. It applies just as easily to \emph{infinite}, structured families of test functions, such as all functions that lie in the unit ball of a vector-valued reproducing kernel Hilbert space. This setting includes, for example, multicalibration with respect to all low-degree polynomials, decision trees, or in the case where $\cY=[0,1]$, all Lipschitz continuous functions of the forecast $p$ \citep{kernelOI,vovk2007k29} (this last case is known as \emph{smooth calibration} \cite{kakade2008deterministic}).

To keep the example concrete, suppose we are interested in building an online forecaster of high-dimensional outcomes that is multicalibrated with respect to all polynomials in $(x,p)$ of degree up to $r \in \mathbb{N}$. Here, $x \in \cX \subseteq \R^m$ and $p \in \cY \subseteq \R^d$. Specifically, consider the feature map $\varphi(x, p)$ listing all monomials in $(x,p)$ of degree up to $r$. Every vector-valued function $h(x,p)$ whose $i$th coordinate is any polynomial in $(x,p)$ of degree up to $r$ can be written as a linear function of $\varphi(x,p)$. That is, $h(x,p)=M \varphi(x,p)$. Assuming a bound on the norm of the coefficients of matrices $M$, we can guarantee multicalibration with respect to the set of functions
\[
    \cH \coloneqq \{(x,p)\mapsto  M \varphi(x, p) : \|M\|_F \le 1\}.
\]

In this case, the set of test functions is indexed by matrices $M$ which span a unit ball. We can then invoke \cref{algo:mc from ol} using a no-regret learner $\cA$ for online optimization over the $\ell_2$ ball (for example, projected gradient descent), and directly obtain an efficient $\sqrt{T}$ multicalibrated forecaster for all low-degree polynomial functions $h$ since the losses in \Cref{algo:mc from ol}, $f_t(h)=f_t(M)$, remain linear in $M$. We will return to this construction in \Cref{sec:mc_implications,sec:applications Phi}, where we illustrate how it extends well beyond low-degree polynomials.

\subsubsection{The Necessity of EVIs: Reducing EVIs to Online Multicalibration}
The previous theorem shows that one can construct a $\cH$-multicalibrated forecaster from a no-regret learner over $\cH$, provided that EVIs can be solved efficiently. Thus, EVIs are \emph{sufficient} to construct multicalibrated forecasters. A natural question, however, is whether they are also the ``right'' primitive for this problem.
In this subsection, we point out that EVIs are necessary as well, in the sense that the existence of an efficient $\cH$-multicalibrated forecaster automatically implies the existence of an efficient EVI solver over functions in $\cH$. In this respect, our result parallels the reduction of \citet{gordon2008no} between external regret and $\Phi$-regret, where (expected) fixed points are sufficient and necessary in the analogous sense \citep{hazan2007computational}.

To show necessity, suppose that we have an $\cH$-multicalibrated forecaster, that is, an algorithm that produces distributions $\cD_t$ such that for any $h\in \cH$ and arbitrary sequence $\{(x_t,y_t)\}_{t=1}^T$,
\[
    \frac{1}{T}\sum_{t=1}^T \E_{p_t\sim\cD_t} \Big[ h(x_t,p_t)^\top (y_t - p_t) \Big] \le \frac{1}{T}\MCerr_T
\]
where $\MCerr_T$ is $o(T)$. 
We show that the above algorithm can be turned into an algorithm for solving EVIs defined in terms of any $x$ and $h$ to any error $\eps>0$:
\[
    \E_{p\sim \cD} \Big[ h(x, p)^\top (y - p) \Big] \le \epsilon \qquad \forall y \in \cY.
\]
To see this, consider the scenario in which Nature picks $x_t = x$ at all times. Let $\cD_t$ be the distribution returned by the $\cH$-calibrated forecaster. We will let the $y_t$ be a best response 
\begin{align*}
y_t \coloneqq\argmax_{\hat{y} \in \cY} \hat{y}^\top \left[\E_{p_t \sim \cD_t} h(x,p_t)\right].
\end{align*}
Then, we have that for any fixed  $y_\star \in \cY$
\begin{align*}
\frac{1}{T}\MCerr_T &\ge \frac{1}{T} \sum_{t=1}^T  \E_{p_t\sim\cD_t} \Big[ h(x, p_t)^\top (y_t - p_t) \Big] 
\ge \frac{1}{T} \sum_{t=1}^T  \E_{p_t\sim\cD_t} \Big[ h(x, p_t)^\top (y_\star - p_t) \Big]  .
\end{align*}
As $T\to\infty$, $\MCerr_T/T\to 0$ and the uniform mixture $\cD$ over $\cD_1, \dots, \cD_T$ is a solution to the EVI defined by $h(x,\cdot)$.

\subsection{$\Phi$-Regret Minimization from Online Multicalibration}
\label{sec:mc to phi}

In this subsection, we provide a \emph{fine-grained} reduction between
$\Phi$-regret minimization and multicalibrated forecasting. The
connection between swap regret (the special case where $\Phi$
contains all endomorphisms on $\cZ$) and best responses to $\ell_1$-calibrated forecasters is
by now folklore, dating at least to \citet{foster1997calibrated}.
However, recent work has shown that swap regret is hard to
minimize efficiently beyond the case where $\cZ$ is the simplex, motivating the study of more structured deviation classes $\Phi$ \citep{daskalakis2024lower}. Furthermore, $\ell_1$ notions of calibration are strictly harder to achieve \cite{collina2026optimal}, motivating the question of whether they are strictly necessary to guarantee low regret.

Our main contribution here establishes that less demanding notions of $\Phi$-regret are efficiently achievable via correspondingly less demanding notions of multicalibration, making this celebrated connection between prediction and decision-making more precise. Specifically, if the goal is to guarantee low regret against a set of deviations $\Phi$, we show that it suffices to best respond to forecasts of the losses that are multicalibrated with respect to the following set of tests:
\begin{equation}
\label{eq:H_PHI_DEF}
    \cH_\Phi \coloneqq \{(x, p) \mapsto \sigma(p) - \phi(\sigma(p)): \phi \in \Phi\}, \qquad \sigma(p) \coloneqq \argmin_{z \in \cZ} p^\top z.
\end{equation}
We assume that ties in the best response problem on the right are broken deterministically according to some arbitrary scheme, so that $\sigma$ is a function. Quite intuitively, as the set of deviations $\Phi$ against which regret is minimized grows, so grows the corresponding set $\cH_\Phi$ of multicalibration tests.
 We discuss in \Cref{sec:applications Phi} how this fine-grained reduction enables efficient guarantees for rich but structured families of deviations, representing a step change in the simplicity of the algorithms for achieving no-$\Phi$-regret against large classes of deviations $\Phi$, and extending naturally to settings in which $\Phi$ is an RKHS ball.

To achieve $\Phi$-regret minimization from $\cH_\Phi$-multicalibration, a decision maker forecasts the next loss (out of the set $\cL$ of losses that could be incurred), and best responds to it. We provide a formal description in \Cref{algo:phi from mc} and prove its correctness in the following theorem which constitutes our second main result.

\begin{algorithm}[t]
    \caption{Black-box reduction from $\cH_\Phi$-multicalibration to $\Phi$-regret minimization}
    \label{algo:phi from mc}
    \DontPrintSemicolon
    \KwData{Convex loss set $\cL\subset \R^d$, convex strategy set $\cZ\subseteq\R^d$, family of deviations $\Phi \subseteq \{\cZ\to \cZ\}$. $\cA$ forecaster for $\cL$ multicalibrated with respect to the set of tests $\cH_\Phi$ defined in \eqref{eq:H_PHI_DEF}}
    \For{$t=1,2,\dots$}{
    Query the forecaster $\cA$ and receive distribution $\cD_t$ over forecasts $p_t \in \cL$\;
    The decision maker chooses the pushforward distribution $\mu_t$ of strategies
    \[
        z_t = \BR(p_t) \coloneqq \argmin_{z \in\cZ}
        ~p_t^\top z, \qquad\qquad \text{where }p_t \sim \cD_t
    \]
    \vspace{-5mm}\;
    Nature reveals the true loss vector $\ell_t \in \cL$\;
    Feed the forecaster the true outcome $y_t = \ell_t$\;}
\end{algorithm}

\begin{theorem}
\label{thm:mc to phi}
Let $\cD_t$ be the sequence of forecast distributions produced in \Cref{algo:phi from mc}, and let $\mu_t$ be the induced distribution over actions $z_t = \BR(p_t)$ where $p_t \sim \cD_t$. Then, the $\Phi$-regret of the resulting decision-making algorithm is upper bounded by the $\cH_\Phi$-multicalibration error of the forecaster. For any $\phi \in \Phi$, if we define $h_\phi(p)= \BR(p) - \phi(\BR(p))$, then:
\[
    \Regret^\Phi_T(\phi) \coloneqq \sum_{t=1}^T \E_{z_t \sim \mu_t}
    \Big[
        \ell_t^\top (z_t - \phi(z_t))
    \Big]
    \le
    \sum_{t=1}^T
    \E_{p_t \sim \cD_t}
    \Big[
        h(p_t)^\top (\ell_t - p_t)
    \Big] = \MCerr_T(h_\phi).
\]
In particular, sublinear multicalibration error with respect to $\cH_\Phi$  by the forecaster, guarantees sublinear $\Phi$-regret by the decision-maker, $\Regret^\Phi_T \leq \MCerr_T$.
\end{theorem}

\begin{proof}
For any $\phi \in \Phi$, by definition of $\mu_t$ as the pushforward of $\cD_t$ under $p \mapsto \BR(p)$,
\begin{align*}
    \sum_{t=1}^T \E_{z_t \sim \mu_t}
    \Big[
        \ell_t^\top (z_t - \phi(z_t))
    \Big] &=
    \sum_{t=1}^T \E_{p_t \sim \cD_t}
    \Big[
        \ell_t^\top \big(\BR(p_t) - \phi(\BR(p_t)) \big)
    \Big] \\
    &=
    \sum_{t=1}^T \E_{p_t \sim \cD_t}
    \Big[
        h_\phi(p_t)^\top (\ell_t - p_t)
    \Big]
    +
    \sum_{t=1}^T \E_{p_t \sim \cD_t}
    \Big[
        h_\phi(p_t)^\top p_t
    \Big]\\
    &\le
    \sum_{t=1}^T \E_{p_t \sim \cD_t}
    \Big[
        h_\phi(p_t)^\top (\ell_t - p_t)
    \Big],
\end{align*}
where in the last inequality we used the fact that for every $p_t$,
$$
    h_\phi(p_t)^\top p_t
    =
     p_t^\top \Big(\BR(p_t) - \phi(\BR(p_t)) \Big)
    \le 0$$
since $\BR(p_t)$ is the minimizer of $p_t^\top z$ over $z$ in $\cZ$. 
\end{proof}
We note that this entire construction also goes through if we add contexts $x_t$ into the picture, yielding an efficient reduction from online multicalibration to \emph{contextual} $\Phi$-regret. In the contextual version, Nature reveals contexts (features) $x_t\in \cX$ before the decision-maker chooses their action $z_t$, and the goal is to minimize 
\begin{align*}
    \sup_{\phi \in \Phi} \sum_{t=1}^T \E_{z_t \sim \mu_t} \Big[\ell_t^\top(z_t - \phi(x_t,z_t))\Big],
\end{align*}
where $\Phi$ is now a set of deviations $\phi: \cX \times \cZ \to \cZ$. In this setting, it suffices for the forecaster to be multicalibrated with respect to the functions $h_\phi(x_t,p_t)= \BR(p_t) - \phi(x_t,\BR(p_t))$. 

As discussed in \Cref{sec:applications Phi}, the simplicity of the reduction, together with the previous result (\Cref{thm:ol to mc}), yields a new path from external to $\Phi$-regret via EVIs and multicalibration. This construction, builds and generalizes on the work of \citet*{noarov2023high} who illustrated a similar connection between forecasting and decision making for the restricted case where the action set $\cZ$ is a finite set of actions and the set $\Phi$ contains all endomorphisms. Our work broadens this connection between forecasting and decision making for general sets of actions $\cZ$ and deviations $\Phi$.

\section{Algorithmic Implications for Online Multicalibration \\and Omniprediction} 
\label{sec:mc_implications}
In this section, we examine some of the consequences of our reduction from regret minimization to online multicalibration (\Cref{thm:ol to mc}). We highlight four main implications.
\begin{enumerate}[label={(\alph*)}]
\item We show that several existing online high-dimensional multicalibration algorithms are special cases of our reduction, thereby unifying their previously disparate analyses. These reinterpretations sometimes yield improved runtimes as well.

\item We demonstrate how the reduction facilitates the technical transfer of ideas between well-studied variants of online learning and online multicalibration, yielding the first known online multicalibration algorithms that succeed in more challenging environments where outcomes are only partially observed. 

\item We show how the reduction yields an oracle-efficient algorithm for online omniprediction in the multiclass setting, solving another open problem from \citet{garg2024oracle}.

\item We show how our reduction can be used to guarantee calibration with respect to more stringent $\ell_1$ or ECE notions of multicalibration at optimal rates.
\end{enumerate}

\subsection{Unified Analyses and Faster Runtimes for Established Algorithms}
\label{sec:defensive}

\paragraph{Defensive Forecasting is Follow The Regularized Leader over the RKHS Ball.} Defensive Forecasting is an influential, game-theoretic approach to sequential prediction introduced by \citet{vovk2005defensive}. Even though it predates the work on multicalibration by \citet*{hebert2018multicalibration}, it achieves exactly the same guarantee for classes $\cH$ that constitute a reproducing kernel Hilbert space or RKHS. Vovk refers to online multicalibration as \textit{resolution} \citep{vovk2007k29}. We point the reader to \citet{perdomo2025defense} for an overview of work in this area.

As background, a vector-valued RKHS is a Hilbert space of functions $\cH \subseteq \{\cX \times \cY \to \R^d\}$ defined in terms of a matrix-valued kernel $\Gamma((\cdot,\cdot),(\cdot,\cdot))$, where $\Gamma((x,p),(x',p')) \in \R^{d \times d}$. Since it is a Hilbert space, $\cH$ is equipped with an inner product $\langle \cdot, \cdot \rangle_{\cH}$ and functions satisfy the reproducing property. For every $h\in \cH$ and $\theta \in \R^d$, $h(x,p)^\top \theta= \langle \Gamma(\cdot,(x,p)) \theta, h \rangle_{\cH}$. 

Here, $\Gamma(\cdot,(x,p))$ is the evaluation functional.
If $\Psi(x,p) \in \R^{r \times d}$ is a fixed feature map, then $\Gamma((x,p),(x',p')) = \Psi(x,p)^\top \Psi(x',p')$ is a matrix-valued kernel. The RKHS $\cH$ in this case is equal to the closure of the set $\{ (x,p) \to \Psi(x,p)^\top \theta: \theta \in \R^r\}$.  

\citet{farina2026defensive} prove that the defensive forecasting procedure in \Cref{alg:defensive_forecasting}, which generalizes the original K29 algorithms from \cite{vovk2005defensive,vovk2007k29} to work efficiently for high-dimensional outcomes, satisfies the following multicalibration guarantee.

\begin{algorithm*}[t]
    \caption{Defensive Forecasting, K29}
    \label{alg:defensive_forecasting}
    \DontPrintSemicolon
    \KwData{Feature map $\Psi: \cX \times \cY \rightarrow \R^{r \times d}$ or matrix-valued kernel $\Gamma((x,p),(x',p'))$}
    \For{$t=1,2,\dots$}{
        Environment reveals $x_t$\;
        Define the operator $S_t:\cY \rightarrow \R^d$ \[S_t(p)
            =
            \sum_{i=1}^{t-1}
            \E_{p_i\sim\cD_i}\!\left[
                \Psi(x_t,p)^\top \Psi(x_i,p_i)(y_i-p_i)  
            \right] = 
            \sum_{i=1}^{t-1}
            \E_{p_i\sim\cD_i}\!\left[
                \Gamma((x_t,p), (x_i,p_i))(y_i-p_i)  
            \right]
        \]\vspace{-5mm}\;
        Output a distribution $\cD_t$ over $p_t\in\cP$ that solves the EVI over $S_t$
        \[
            \E_{p_t\sim\cD_t}\!\left[S_t(p_t)^\top (y-p_t)\right]
            \le \eps_t
            \qquad \forall y\in\cP
        \]\vspace{-5mm}\;
        Environment reveals $y_t$\;
    }
\end{algorithm*}

\begin{algorithm*}[t]
    \caption{Reduction (\Cref{thm:ol to mc}, \Cref{algo:mc from ol}) with RFTL over the RKHS ball}
    \label{alg:rftl_reduction}
    \DontPrintSemicolon
    \KwData{Feature map $\Psi: \cX \times \cY \rightarrow \R^{r \times d}$}
    Initialize $\theta_1=0$\;
    \For{$t=1,2,\dots$}{
        Environment reveals $x_t$\;
        RFTL algorithm $\theta_t$, defining $h_t(x,p)=\Psi(x,p)^\top \theta_t$\;
    
        Output a distribution $\cD_t$ over $p_t\in\cP$ such that\;
        \[
            \E_{p_t\sim\cD_t}\!\left[h_t(x_t,p_t)^\top (y-p_t)\right]
            \le \eps_t
            \qquad \forall y\in\cY
    \]\vspace{-5mm}\;
Environment reveals $y_t$, set
            $\ell_t
            =
            \E_{p_t\sim\cD_t}\!\left[\Psi(x_t,p_t)(y_t-p_t)\right]$\;
    
        Update
        \[
            \theta_{t+1}
            \in
            \argmin_{\|\theta\|_2\le B}
            \left\{
                \eta \sum_{i=1}^{t} -\,\ell_i^\top\theta + \|\theta\|_2^2
            \right\}
        \]
    }
\end{algorithm*}

\begin{proposition}[\cite{farina2026defensive}]\label{prop:defensive}
Let $\cH \subseteq \{\cX \times \cY \to \R^d\}$ be any vector-valued RKHS with matrix-valued kernel $\Gamma((x,p),(x',p'))$. Then, for any function $h$ in the RKHS $\cH$, if $\eps_t= 1/(10t^2)$, \Cref{alg:defensive_forecasting} produces distributions $\cD_t$ such that for any sequence of $(x_t,y_t)$,
\begin{align*}
\bigg|    \sum_{t=1}^T \E_{p_t \sim \cD_t} \bigr[ h(x_t,p_t)^\top (y_t - p_t) \bigr] \bigg| \leq \|h\|_{\cH} \sqrt{\sum_{t=1}^T\E_{p_t \sim \cD_t}(y_t - p_t)^\top\Gamma((x_t,p_t), (x_t,p_t))(y_t - p_t) + 1},
\end{align*}
where $\|h\|_{\cH}^2 = \langle h, h \rangle_{\cH}$. In the case where $\Gamma((x,p),(x',p'))= \Psi(x,p)^\top \Psi(x',p')$ for some feature mapping $\Psi(x,p) \in \R^{r \times d}$, the guarantee above can be rewritten as: For any $h_\theta(x,p)= \Psi(x,p)^\top \theta$, 
\begin{align*}
\bigg|    \sum_{t=1}^T \E_{p_t \sim \cD_t} \bigr[ h_\theta(x_t,p_t)^\top (y_t - p_t) \bigr] \bigg| \leq \|\theta\|_{2} \sqrt{\sum_{t=1}^T\|\E_{p_t \sim \cD_t}\Psi(x_t,p_t)(y_t - p_t)\|_2^2 + 1}.
\end{align*}
At time $t$ the algorithm runs in time $\cO(\poly(\log(tD),\mathrm{eval}(\Gamma)))$ where $D = \sup_{(x,p)} \|\Gamma((x,p),(x,p))\|_{\mathrm{op}}$ and $\mathrm{eval}(\Gamma)$ is an upper bound on the time it takes to evaluate the kernel function $\Gamma$.
\end{proposition}

Here, we show that Defensive Forecasting is just a special instantiation of our general reduction with follow-the-regularized-leader (FTRL) over the $\ell_2$ ball as the regret minimizer. We can therefore reprove its correctness using the following textbook result on FTRL and demystify the origin of the operators $S$ in \Cref{alg:defensive_forecasting}.

\begin{lemma}[Theorem 5.2 in \cite{hazan2016introduction}]
\label{lemma:rftl}
Consider the follow-the-regularized-leader algorithm for online linear optimization over the ball of radius $B$ where $\theta_1=0$ and for all $t>1$,
\begin{align*}
    \theta_t = \argmin_{\theta: \|\theta\|_2\leq B} 
    \left\{\eta \sum_{i=1}^{t-1} -\ell_i^\top \theta + \|\theta\|_2^2 \right\}.
\end{align*}
Then, there exists a choice of $\eta$ such that for all $t$, the vectors $\theta_t$ can be written as
\begin{align}
\label{eq:rflt_theta}
\theta_t= \alpha_t \sum_{i=1}^{t-1} \ell_i,
\end{align}
where $\alpha_t$ is a positive scalar, and for any $\theta$ with $\ell_2$ norm  bounded by $B$, 
\begin{align}
\label{eq:rflt_regret}
    \sum_{t=1}^T \ell_t^\top \theta - \sum_{t=1}^T \ell_t^\top \theta_t \leq 2B\sqrt{\sum_{t=1}^T\|\ell_t\|^2}.
\end{align}
\end{lemma}
Plugging in the choice $\theta_t$ from FTRL in \Cref{eq:rflt_theta} into the expression for $h_t$ from \Cref{alg:rftl_reduction}, and using that $\ell_i = \underset{p_i \sim \cD_i}{\E}\bigl[\Psi(x_i, p_i)(y_i - p_i)\bigr]$ we get that:
\[
  h_t(x_t, p) = \Psi(x_t, p)^\top \theta_t = \alpha_t \cdot \Psi(x_t, p)^\top \sum_{i=1}^{t-1} \underset{p_i \sim \cD_i}{\E}\bigl[\Psi(x_i, p_i)(y_i - p_i)\bigr]= \alpha_t S_t(p).
\]
Since $\alpha_t > 0$ is a positive scalar, the operators over which we are solving EVIs in \Cref{alg:rftl_reduction}, $h_t(x_t, \cdot)$, and in \Cref{alg:defensive_forecasting}, $S_t$, are the \textit{same} up to scaling. For every $p$, $h_t(x_t, p) = \alpha_t \cdot S_t(p)$. Therefore, if $\cD_t$ is an EVI solution over the operator $h_t(x_t, \cdot)$ with error $\eps$, then $\cD_t$ is also an EVI solution over $S_t$ but with error $\eps / \alpha_t$.
Hence, the distributions $\cD_t$ that both algorithms select at every time step are essentially identical.

It remains to verify that the two algorithms also have matching multicalibration errors with respect to all functions of the form $h_\theta(x,p)=\Psi(x,p)^\top\theta$.
Setting $\eps_t = 1/(10t^2)$ in \Cref{alg:rftl_reduction} gives $\mathrm{EVI}_T = \sum_{t=1}^T \eps_t \leq 1$.
For the regret term, \Cref{lemma:rftl} yields
\begin{align*}
    \Regret_T
    \leq 2B\sqrt{\sum_{t=1}^T \|\ell_t\|^2} 
    &= 2B\sqrt{\sum_{t=1}^T \Big\|\E_{p_t \sim \cD_t}\bigl[\Psi(x_t, p_t)(y_t - p_t)\bigr]\Big\|^2_2}\\
    &\le 2B\sqrt{\sum_{t=1}^T \E_{p_t \sim \cD_t}\Big[(y_t - p_t)^\top \Gamma((x_t,p_t),(x_t,p_t))(y_t - p_t)\Big]},
\end{align*}
where the last equality uses $\Gamma((x,p),(x',p')) = \Psi(x,p)^\top \Psi(x',p')$ and Jensen's inequality.
Applying \Cref{thm:ol to mc}, the multicalibration error of the distributions produced by \Cref{alg:rftl_reduction} with respect to any $h_\theta(x,p) = \Psi(x,p)^\top \theta$ with $\|\theta\|_2 \leq B$ satisfies
\begin{align*}
    \MCerr_T(h_\theta)
    \leq \Regret_T + \mathrm{EVI}_T
    \leq 2B\sqrt{\sum_{t=1}^T \E_{p_t \sim \cD_t}\Big[(y_t - p_t)^\top \Gamma((x_t,p_t),(x_t,p_t))(y_t - p_t)\Big]} + 1,
\end{align*}
matching the Defensive Forecasting guarantee of \citet{farina2026defensive} up to a small constant.

\paragraph{Cubically Faster Runtimes for the Unbiased Prediction Algorithm} \citet{noarov2023high} introduced the first algorithm for online high-dimensional multicalibration. It guarantees multicalibration with respect to any function in the convex hull of a finite collection of functions $\cH$. We present pseudocode for the unbiased prediction algorithm in \Cref{algo:noarov}.

Here, we provide a simplified proof of its correctness via \Cref{thm:ol to mc} and improve its runtime from $\Omega(T^3)$ to $\widetilde{O}(T)$. In particular, the procedure exactly matches the template of our reduction with a specific choice of regret minimizer and EVI solver (it is, in effect, an improved version of the Hedge procedure we presented in \Cref{sec:reductions}). We can reprove its correctness by simply plugging in the regret bound from \citet{chen2021impossible} into \Cref{thm:ol to mc}.

\begin{algorithm}[t]
    \caption{The Unbiased Prediction Algorithm \citep{noarov2023high}}
    \label{algo:noarov}
    \DontPrintSemicolon
    \For{$t=1,2,\dots$}{
        Nature reveals $x_t$\;
        Run the Multiscale Multiplicative Weights with Correction (MMWC) algorithm to receive $h_t$\;
Solve the following EVI induced by $h_t$ using a specific no-regret algorithm to $\eps_t =1/t$      
\[
            \E_{p\sim \cD_t}\!\left[h_t(x_t,p)^\top (y-p)\right] \le \eps_t
            \qquad \forall y\in\cY
        \]  \vspace{-5mm}\;
        Nature reveals $y_t$\;
        Feed the MMWC regret minimizer the loss vector $f_t\in\R^{|\cH|}$ with coordinates
        \[
            f_{t,h}
            \coloneqq
            -\,\E_{p_t\sim\cD_t}\!\left[h(x_t,p_t)^\top (y_t-p_t)\right]
        \]
    }
\end{algorithm}

\begin{corollary}
\label{corr:noarov}
The Unbiased Prediction Algorithm produces distributions such that for any $h\in \cH$, 
\begin{align*}
	\sum_{t=1}^T \E_{p_t\sim \cD_t} h(x_t,p_t)^\top(y_t-p_t) \leq \widetilde{\cO}\left( \sqrt{ \log(|\cH|) \cdot \sum_{t=1}^T \E_{p_t\sim\cD_t}\Big[ h(x_t,p_t)^\top (y_t - p_t) \Big]^2}\right).
\end{align*}
\end{corollary}

\begin{proof}
\citet{chen2021impossible} prove that for any sequence of losses $\ell_t$, the Multiscale Multiplicative
Weights with Correction algorithm produces $h_t$ such that for any $h\in \cH$,
\begin{align}
\label{eq:regret_mmm}
	\Regret_T(h) = \sum_{t=1}^T f_t(h_t) - \sum_{t=1}^T f_t(h) \leq \widetilde{\cO}
	\left(\sqrt{\log(|\cH|)\sum_{t=1}^T f_t(h)^2}\right).
\end{align}

By \Cref{thm:ol to mc}, the multicalibration error is bounded by $\Regret_T(h)+ \mathrm{EVI}_T$. Since \cite{noarov2023high} set $\eps_t\leq 1/t$, then $\mathrm{EVI}_T = \sum_{t=1}^T \eps_t = \sum_{t=1}^T t^{-1} \leq \cO(\log(T))$. The stated regret bound follows from substituting in the fact that $f_t(h)= -\E_{p_t\sim\cD_t}\Big[ h(x_t,p_t)^\top (y_t - p_t) \Big]$.
\end{proof}
As a comment, here we present the regret bound from \cite{noarov2023high} in a slightly different form, where we let $\cH \subseteq \{\cX \times \cY \to \R^d\}$ be any finite set of functions. To recover the exact bound in \cite{noarov2023high}, we can pick $\{(x,p) \to s \cdot \vec{e_i} \cdot c(x,p): c \in \cC, i \in [d], s \in \{\pm 1\}\}$ where $\cC \subseteq \{\cX \times \cY \to \{0,1\}\}$ is a finite set of binary functions and $\vec{e}_i$ denotes the $i$th standard basis vector in $\R^d$.

Recall that to solve EVIs via regret minimization algorithms, we need to run the regret minimizer for $k$ such rounds where $k$ satisfies $\Regret(k)/k\leq \eps_t$. Setting $\epsilon_t =1/t$ as in \cite{noarov2023high}, since $\Regret(k)$ is $\Theta(\sqrt{k})$, we need to set $k \geq \Omega(t^2)$. The total runtime of the algorithm is then lower bounded by $\sum_{t=1}^T \Omega(t^2) \geq \Omega(T^3)$. However, by swapping out the EVI solver for the ellipsoid-based procedure from \cite{zhang2025expected}, we can improve the runtime to $\cO(T\log T)$. We summarize this statement formally below. 

\begin{proposition}
Fix a finite class of functions $\cH \subseteq \{\cX \times \cY \to \R^d\}$ and consider the version of the unbiased prediction algorithm where the EVI solving procedure in line 4 is replaced with the ellipsoid algorithm from \cite{zhang2025expected} which solves EVIs to precision $\epsilon$ in time $\log(1/\eps)$.
Then this modified procedure achieves an identical $\MCerr_T$ for every $h\in \cH$ to that presented in \Cref{corr:noarov}, yet the runtime over $T$ rounds is $\widetilde{\cO}(T \cdot \poly(|\cH|))$ instead of $\Omega(T^3\cdot \poly(|\cH|))$ as in \Cref{algo:noarov}.
\end{proposition}

\subsection{New Online Multicalibration Algorithms for Complex Environments}

We move on to the second set of implications of \Cref{thm:ol to mc}, showing how it directly enables online multicalibration algorithms with provable guarantees in richer environments.

\subsubsection{Online Multicalibration with Delayed Observations}
One immediate implication of \Cref{thm:ol to mc} is that it allows us to design algorithms which guarantee online multicalibration even when outcomes $y_t$ are not immediately observed. 

To motivate the problem, consider the setting where at every day $t$, a doctor predicts the likelihood of a patient having a particular disease, yet the presence or absence of this disease is not revealed until the patient comes for a follow-up visit at some future date $t'=t+d_t$ where $d_t$ is a positive integer delay. Even though the outcome is not observed for another $d_t$ days, the doctor still needs to make predictions for all patients who arrive between $t$ and $t'$. Furthermore, different patients might have different visit schedules, hence the delays $d_{t}$ are unknown and can vary with $t$. 

Despite these difficulties, the doctor aims to make predictions that are not just calibrated, but multicalibrated without making any assumptions about how the data is generated. That is, they aim to sequentially produce predictions that satisfy the following definition:

\begin{definition}[Online Multicalibration under Delayed Observations]
\label{def:mc_delays}
Fix $\cH \subseteq \{\cX \times \cY \to \R^d\}$ and consider the following online protocol where at every time $t$, Nature selects features $x_t$ and reveals previous outcomes $\{y_i: i+ d_i = t, d_i > 0\}$ to the Learner. Then, the Learner outputs a distribution $\cD_t$ over predictions $p_t$.

An algorithm guarantees online multicalibration under delays with respect to $\cH$ if it selects $\cD_t$ such that regardless of the feature, outcome pairs $(x_t,y_t)$ and delays $d_t$ chosen by Nature, for any $h\in \cH$,
\[
   \MCerr_T(h) =  \sum_{t=1}^T \E_{p_t\sim\cD_t} \Big[ h(x_t,p_t)^\top (y_t - p_t) \Big] = o(T).
\]
\end{definition}
Note that this setup generalizes the standard online setting studied so far where the Learner immediately sees the outcomes $y_t$ after selecting $\cD_t$. In this case, $d_t=1$ for all $t$. To the best of our knowledge, our work initiates the study of this problem. Using our reduction, we can design simple algorithms that achieve this guarantee by leveraging the fact that there are by now many online learning algorithms that can handle delayed losses as per the following definition.

\begin{definition}
\label{def:ol_delays}
Fix a set $\cH$ and consider the protocol where at every round $t$, Nature reveals a set of prior losses $\{\ell_i: i+ d_i = t\}$ to the Learner, who then selects $h_t$. An algorithm guarantees no-regret in this online learning with delays protocol if
\[
   \Regret_T(h) \coloneqq \sum_{t=1}^T \ell_t(h_t) - \sum_{t=1}^T \ell_t(h)  = o(T)
\]
for all $h \in \cH$.
\end{definition}
We show that our main reduction between regret minimization and online multicalibration handles this extension seamlessly. The key insight is that \Cref{algo:mc from ol} has two main parts: the EVI solver and the online Learner that produces $h_t$. Even though losses are not observed until a later time step, the online Learner in \Cref{def:ol_delays} still produces a solution $h_t$ at every time $t$, and this is enough to produce the distribution $\cD_t$ since the EVI solver only needs $h_t$.

\begin{corollary}
Let $\cA$ be any algorithm that guarantees $\Regret_T$ in the online learning with delays protocol with respect to $\cH$ as per \Cref{def:ol_delays}. Then, using $\cA$ as the external regret minimizer in \Cref{algo:mc from ol} guarantees online multicalibration with delays as per \Cref{def:mc_delays}. 
For any $h\in\cH$,
\begin{align*}
    \MCerr_T(h) \leq \Regret_T(h) + \EVI_T.
\end{align*}
\end{corollary}

As a specific example consider the case where $h_\theta(x_t,p_t) = \Psi(x_t,p_t)^\top \theta$ (for simplicity, we will use $h_\theta$ and $\theta$ interchangeably in this derivation) for some fixed feature map $\Psi(x,p) \in \R^{r \times d}$. \citet{quanrud2015online} proved that there exists a choice of constant step size $\eta$ such that the following online gradient descent procedure, where we update the parameters asynchronously on the basis of the losses observed at time $t$, $\cL_t = \{\ell_s: s+ d_s =t\}$, 
\begin{align*}
    \theta_{t+1} = \Pi_{\mathbb{B}_2}\mleft[\theta_t - \eta \sum_{\ell_s \in \cL_t} \nabla \ell_s(\theta_s)\mright],
\end{align*}
satisfies the following regret bound for any $\theta$ with $\ell_2$ norm at most 1:
\begin{align*}
   \sum_{t=1}^T \ell_t(\theta_t) -  \sum_{t=1}^T \ell_t(\theta) \leq  \mathcal{O}\left(\max_{1\leq t \leq T} \|\nabla \ell_t(\theta_t)\|_2\sqrt{T + \sum_{t=1}^T d_t}\right).
\end{align*}

Setting $\eps_t= 1/\sqrt{t}$ and bounding $\EVI_T \leq \cO(\sqrt{T})$, this implies the following end-to-end bound for online multicalibration with delays for any $h_\theta= \Psi(x,p)^\top \theta$ where $\theta$ is in the unit ball:
\begin{align*}
    \MCerr_T(h_\theta) \leq \mathcal{O}\left(\max_{1\leq t \leq T} \|\E_{p_t\sim \cD_t} \Psi(x_t,p_t)(y_t-p_t)\|_2\sqrt{T + \sum_{t=1}^T d_t}\right).
\end{align*}
The bound is $o(T)$ as long as $\sum_{t=1}^T d_t$ is $o(T^2)$.

\subsubsection{Online Multicalibration under Censored Outcomes}

In many real-world decision-making settings, the true outcome $y_t$ is censored: we observe $y_t$ only if the predictions $p_t$ fall within a specific range. For example, in lending, we only observe whether people repay loans if we issue them loans, a decision that occurs only if the predicted likelihood of repayment $p_t$ is not too small.

Here, we slightly modify our reduction to design the first known algorithms for online multicalibration under censoring. The main modification we need is to add importance sampling into our reduction, as in prior work on online learning (e.g., \citep{cesa2005label}). We start by formally defining the online multicalibration with censored observations problem:

\begin{definition}[Online Multicalibration under Censored Observations]\label{def:censored}
  Fix a class of functions $\cH \subseteq \{\cX \times \cY \to \R^d\}$ and a censoring function $C: \Delta(\cY) \to \{0,1\}$. Consider the online protocol where at every time $t$, Nature selects $(x_t,y_t) $ and then reveals $x_t$ to the Learner who outputs a distribution $\cD_t$ over predictions $p_t$. Lastly, Nature reveals $y_t$ to the Learner if $C(\cD_t)=1$.

  An algorithm guarantees online multicalibration under censored observations with respect to $\cH$ if, regardless of Nature's choices, the following holds for every $h\in \cH$,
  \[
    \MCerr_T(h)= \sum_{t=1}^{T} \E_{p_t \sim \cD_t}\left[ h(x_t, p_t)^\top (y_t - p_t) \right] = o(T).
  \]
\end{definition}

We note that in the definition, the multicalibration error is summed over all rounds $T$, regardless of whether the Learner observed the outcome $y_t$ in any given round $t$. Also, unlike in previous versions, $y_t$ is selected non-adaptively: Nature selects $y_t$ before the Learner chooses the distribution $\cD_t$. Without loss of generality (since the Learner knows the censoring function $C(\cdot)$), we will assume that there is a known $\cD_\star$ such that if $\cD_t= \cD_\star$, then $C(\cD_\star)=1$ and $y_t$ is observed. 

\begin{algorithm*}[t]
    \caption{A General Reduction for Online Multicalibration with Censored Outcomes}
    \label{alg:censored_reduction}
    \DontPrintSemicolon
    \KwData{No-regret learner $\cA$; exploration distribution $\cD^*$; censoring parameter $\gamma\in(0,1]$}
    \For{$t=1,2,\dots$}{
        Nature reveals $x_t$\;
        Online learner selects $h_t \in \cH$\;
        Find a distribution $\cD_t$ that solves the EVI over $h_t(x_t,\cdot)$ to precision $\epsilon_t$, that is,
        \[
            \E_{p_t \sim \cD_t}\!\left[
                h_t(x_t,p_t)^\top (y-p_t)
            \right]
            \le \epsilon_t
            \qquad \forall y\in\cY
        \]\vspace{-5mm}\;
        Sample $Z_t \sim \mathrm{Bernoulli}(\gamma)$. If $Z_t = 1$,  output $\cD^*$ and observe $y_t$. If $Z_t = 0$,  output $\cD_t$.\;
        The effective output distribution is $\widetilde{\cD}_t = (1-\gamma)\cD_t + \gamma \cD^*$\;
        Feed the regret minimizer the importance-weighted loss
        \[
            \hat{f}_t(h)
            =
            -\frac{Z_t}{\gamma}\,
            \E_{p \sim \cD_t}\!\left[
                h(x_t,p)^\top (y_t-p)
            \right]
        \]\vspace{-6mm}\;
    }
\end{algorithm*}

\begin{proposition}
Let $\cA$ be any online algorithm that selects $h_t$ such that for any $h \in \cH, \gamma >0$, and sequence of losses $\hat{f}_t$ that are linear in $h$ and $1/\gamma$
\begin{align*}
   \Regret_{T,1/\gamma} \coloneqq \sup_{h\in \cH} \sum_{t=1}^T \hat{f_t}(h_t) - \sum_{t=1}^T \hat{f}_t(h) = o(T) .
\end{align*}
Then, using $\cA$ as the no-regret learner in \Cref{alg:censored_reduction} satisfies the following online multicalibration under censored observations error (\Cref{def:censored}) with respect to functions in $\cH$,
\begin{align*}
\MCerr_T \leq \gamma L T + \EVI_T + \E_{Z_1, \dots, Z_T}[\Regret_{T,1/\gamma}].
\end{align*}
Here, $L$ is a uniform upper bound on $|h(x_t,p_t)^\top(y_t-p_t)|$.
\end{proposition}

\begin{proof}
  Fix any $h \in \cH$. By definition of $\widetilde{\cD}_t$,
  \begin{align*}
    \sum_{t=1}^{T} \E_{p_t \sim \widetilde{\cD}_t} \left[h(x_t, p_t)^\top(y_t - p_t)\right] &= (1-\gamma)\sum_{t=1}^{T} \E_{p_t \sim {\cD}_t} \left[h(x_t, p_t)^\top(y_t - p_t)\right] \\
    &\hspace{3cm} + \gamma\sum_{t=1}^{T} \E_{p_t \sim {\cD_\star}} \left[h(x_t, p_t)^\top(y_t - p_t)\right].
  \end{align*}
By assumption, the second term is bounded by $\gamma L T$ and by the EVI guarantee, 
  \begin{align*}
      0 \leq \eps_t - \E_{p_t \sim \cD_t}[h_t(x_t, p_t)^\top(y_t - p_t)].
  \end{align*}
Using these two facts, we get that $\MCerr_T(h)$ is at most
\begin{align*}
    \gamma L T + \sum_{t=1}^T\eps_t + \sum_{t=1}^T \E_{p_t \sim {\cD}_t} \left[h(x_t, p)^\top(y_t - p_t)\right] - \E_{p_t \sim {\cD}_t} \left[h_t(x_t, p_t)^\top(y_t - p_t)\right].
\end{align*}
Next, we observe that for any $h$, $\E_{Z_t}[\hat{f}(h)] = -\E_{p_t \sim {\cD}_t} \left[h(x_t, p_t)^\top(y_t - p_t)\right]$. Therefore, using linearity of expectation, we can rewrite the above as 
\[
    \gamma L T + \sum_{t=1}^T\eps_t + \sum_{t=1}^T \E_{Z_t}[\hat{f}_t(h_t) - \hat{f}_t(h)] = \gamma L T + \EVI_T + \E_{Z_1, \dots, Z_T}[\Regret_{T, 1/\gamma}]. \qedhere
\]
\end{proof}
Since $\Regret_{T,1/\gamma}$ typically scales linearly with $1/\gamma$, by tuning $\gamma$, we get a sublinear regret bound. As an example, if we run the Multiplicative
Weights  algorithm, we get that 
\begin{align*}
\Regret_{T,1/\gamma} \leq \widetilde{\cO}
	\left(\sqrt{\log(|\cH|) T L^2/\gamma^2}\right).
\end{align*}
Optimizing over $\gamma$ to balance between terms, we get a $\widetilde{\cO}(T^{3/4})$ rate for this problem by running Multiplicative Weights (Hedge) as the online Learner.

\begin{remark}[Dynamic Multicalibration]
This list of consequences is by no means exhaustive. For instance, we could also guarantee dynamic multicalibration regret bounds of the form
\begin{align*}
    \sum_{t=1}^T \E_{p_t\sim \cD_t} h_t(x_t,p_t)^\top(y_t-p_t) = o(T),
\end{align*}
where one can specify the sequence of functions $\{h_t\}_{t=1}^T$ (that change over time, yet have bounded path length) by plugging in dynamic, path-dependent regret bounds into our reduction.
\end{remark}

\subsection{Oracle-Efficient Algorithms for Online, Multiclass Omniprediction}

To close this section, we show how by leveraging our reduction together with known connections between multicalibration and \emph{omniprediction} \cite{omnipredictors,gopalan2022loss}, we can derive simple algorithms for multiclass omniprediction that are efficient whenever online learning is efficient. 

We start by defining online omniprediction for multiclass outcomes $y$. For $p_t$ in the simplex, we write $\yt_t \sim p_t$ to denote outcomes $\yt_t$ sampled from the conditional distribution specified by $p_t$.\footnote{If $p_t$ is in the simplex, then we can write $p_{t,i} = \Pr[\yt_t=i]$ for every $i = 1,\dots,k$.} Furthermore, for multiclass outcomes $y$, we adopt the notation $\vec{y}=(1\{y=1\}, \dots, 1\{y=k\})$ to denote the one-hot encoded version of $y$.
\begin{definition}
Consider the protocol in which, at each time step, Nature reveals features $x_t$, the Learner chooses a distribution $\cD_t$ over forecasts $p_t \in \Delta_k = \{p \in \R^k: p_i \geq 0, \sum_{i=1}^k p_i = 1\}$, and Nature reveals the true outcome $y_t \in [k] = \{1, \dots, k \}$.

Fix a set of actions $\cZ$. Let $\cL \subseteq \{\cZ \times [k] \to \R_{\geq 0} \}$ be a set of loss functions and $\cC \subseteq \{\cX \to \cZ\}$ be a set of decision rules. An online algorithm is an omnipredictor with respect to $\cL$ and $\cC$ if it produces distributions $\cD_t$ such that for every $\ell \in \cL$ and $\pi_\ell(p_t) = \argmin_{z \in \cZ} \E_{\yt_t \sim p_t} \ell(z, \yt_t)$,
\begin{align*}
\sum_{t=1}^T \E_{p_t\sim \cD_t} \left[\ell(\pi_\ell(p_t), y_t) \right] \leq \min_{c \in \cC} \sum_{t=1}^T\ell(c(x_t),y_t) + o(T).
\end{align*}
\end{definition}
Intuitively, an omnipredictor is an algorithm that produces predictions that can be \emph{postprocessed} to yield good decisions. They guarantee that the best response function to any function $\ell$,
\begin{align*}
\pi_\ell(p_t) = \argmin_{z \in \cZ} \E_{\yt_t \sim p_t} \ell(z, \yt_t) = \argmin_{z \in \cZ} \sum_{i=1}^k \ell(z,i) \cdot p_{t,i},
\end{align*}
has low excess risk to the best decision rule in $\cC$ for that loss $\ell$. Note that by definition, the postprocessing of $p_t$ has lower loss than \textit{any} decision rule $c: \cX \rightarrow \cZ$ on outcomes sampled from $p_t$: 
\begin{align*}
    \E_{\yt \sim p} [\ell(\pi_\ell(p_t),\yt)] \leq  \E_{\yt \sim p} [\ell(h(x),\yt)].
\end{align*}
It is by now well known \cite{omnipredictors,garg2024oracle,gopalan2022loss,kernelOI,lu2025sample} that given any set of loss functions $\cL$ and decision rules $\cC$, one can construct a set of test functions $\cH_{\cL, \cC}$ (defined in terms of $\cL$ and $\cC$) such that multicalibration with respect to $\cH_{\cL, \cC}$ implies online omniprediction with respect to $\cL$ and $\cC$. 

What is new in our work is that by exploiting this connection, and our new path between regret minimization and online multicalibration, we establish a path between regret minimization and omniprediction that yields oracle-efficient algorithms for omniprediction. 

We start by stating the following (now folklore) lemma, which we specialize to the multiclass setting, proving that multicalibration implies omniprediction.
 
\begin{lemma}
\label{lemma:general_loss_min}
Fix a class of loss functions $\cL \subseteq \{\cZ \times [k] \to \R_{\geq 0} \}$ and a set of decision rules $\cC \subseteq \{\cX \to \cZ\}$. Define the set of test functions 
\begin{align}
\label{eq:omni_tests}
\cH_{\cL,\cC} = \{h_\ell: \ell \in \cL\} \cup \{h_{\ell, c}: \ell \in \cL, c \in \cC\} \subseteq \{\cX \times \Delta_k \to \R^d\}
\end{align}
 where $h_\ell$ and $h_{\ell,c}$ are defined as:
\begin{align*}
h_\ell(p_t) &\coloneqq \big(\ell(\pi_\ell(p_t),1), \dots, \ell(\pi_\ell(p_t),k)\big)^\top \in \R^k \\
h_{\ell,c}(x_t) &\coloneqq -\big(\ell(c(x_t),1), \dots, \ell(c(x_t),k)\big)^\top \in \R^k. 
\end{align*}
If an online algorithm $\cA$ produces distributions $\cD_t$ with multicalibration error bounded by $\MCerr_T$ with respect to $\cH_{\cL, \cC}$, then $\cA$ is an omnipredictor with respect to $\cL$ and $\cC$ with omniprediction error bounded by $2 \MCerr_T$. That is, for every $\ell 
\in \cL$,
\begin{align*}
\sum_{t=1}^T \E_{p_t\sim \cD_t} \left[\ell(\pi_\ell(p_t), y_t) \right] \leq \min_{c \in \cC} \sum_{t=1}^T\ell(c(x_t),y_t) + 2\MCerr_T.
\end{align*}
\end{lemma}

\begin{proof}
The proof follows the same template pioneered by \citet*{gopalan2022loss}, showing that outcome indistinguishability implies omniprediction. In particular, using the notation $\yt_t \sim p_t$ to denote sampling a multiclass outcome from the distribution specified by $p_t \in \Delta_k$ where $\Pr[\yt=i]=p_{t,i}$, we have the following rewriting,
\begin{align}
\label{eq:dec_oi_h}
    \sum_{t=1}^T \underset{p_t \sim \cD_t}{\E}[\ell(\pi_\ell(p_t), y_t)] - \sum_{t=1}^T \underset{\tilde{y}_t \sim p_t}{\underset{p_t \sim \cD_t}{\E}}[\ell(\pi_\ell(p_t), \yt_t)] &= \sum_{t=1}^T \E_{p_t\sim \cD_t} \left[h_\ell(p_t)^\top(\vec{y}_t -p_t)\right] \leq \MCerr_T  \\ 
    \sum_{t=1}^T \underset{\tilde{y}_t \sim p_t}{\underset{p_t \sim \cD_t}{\E}}[\ell(c(x_t), \yt_t)] - \sum_{t=1}^T \ell(c(x_t), y_t) &= \sum_{t=1}^T \E_{p_t \sim \cD_t} \left[h_{\ell,c}(x_t)^\top (\vec{y}_t - p_t) \right]\leq \MCerr_T. \label{eq:loss_oi_h}
\end{align}
Using these inequalities, for any function $\ell \in \cL$ and $c \in \cC$, we have that,
\begin{align*}
\sum_{t=1}^T \E_{p_t\sim \cD_t}[\ell(\pi_\ell(p_t), y_t)] &\leq \sum_{t=1}^T \underset{\tilde{y}_t \sim p_t}{\underset{p_t \sim \cD_t}{\E}}
\bigl[\ell(\pi_\ell(p_t), \tilde{y})\bigr] + \MCerr_T  &\text{ (By \Cref{eq:dec_oi_h})}\\
& \leq \sum_{t=1}^T \underset{\tilde{y}_t \sim p_t}{\underset{p_t \sim \cD_t}{\E}}[\ell(c(x_t), \yt_t)] + \MCerr_T &\text{(By definition of $\pi(p_t)$)}  \\ 
& \leq \sum_{t=1}^T \ell(c(x_t), y_t) + 2\MCerr_T\,. &\text{(By \Cref{eq:loss_oi_h})}
\end{align*}
Since $\ell$ and $c$ were arbitrary, the argument holds for the best $c \in \cC$, thereby proving the lemma.
\end{proof}

Using this lemma, and our black-box reduction between multicalibration and online learning \Cref{thm:ol to mc}, we get the following oracle-efficient algorithm for omniprediction.

\begin{corollary}
Fix a class of loss functions $\cL$ and decision rules $\cC$. Define the class of test functions $\cH_{\cL, \cC}$ as in \Cref{eq:omni_tests} and let $\cA$ be an online algorithm that guarantees external regret at most $\Regret_T$ with respect to $\cH_{\cL, \cC}$, as per \Cref{eq:ext_regret}. 
Then, if the EVI errors are $\epsilon_t$, using $\cA$ in \Cref{algo:mc from ol} yields omniprediction error bounded by $2\Regret_T + \sum_t \epsilon_t$. 
\end{corollary}

In short, this corollary is an end-to-end reduction between omniprediction and online learning. Depending on the pair $(\cL, \cC)$, there are many online algorithms one could use to guarantee omniprediction. For instance, one could always run Multiplicative Weights over finite classes $(\cL, \cC)$ and get an online omniprediction algorithm with error bounded by $\cO(\sqrt{T\log(|\cL| \cdot |\cC|)})$. If $(\cL, \cC)$ belongs to an RKHS, one could run defensive forecasting algorithms (see, e.g., \cite{kernelOI,farina2026defensive}) and get a $\sqrt{T}$ bound with additional dependence on the norms of the functions in the RKHS.

\citet{okoroafor2025near} present oracle efficient algorithms for the specific case of omniprediction with respect to binary outcomes. Relative to their work, our results hold for general multiclass settings. Furthermore, our construction can also be extended to work for general real-valued or vector-valued outcomes by using ideas in \cite{lu2025sample,gopalan2024omnipredictors} where one rewrites nonlinear loss functions as linear functions of a high-dimensional vector of statistics. We omit this construction to keep our presentation short.

\subsection{Optimal Rates for (Swap) ECE Multicalibration}
\label{sec:ece}

A large fraction of the calibration literature has focused on achieving what is known as low expected (multi)calibration error, which we refer to as $\ell_1$ (or ECE) multicalibration. Relative to our notion, where we define a general class of functions $\cH \subseteq \{ \cX \times \cY \rightarrow \R^d\}$ and guarantee that
\begin{align*}
\sup_{h \in \cH} \sum_{t=1}^T \E_{p \sim \cD_t}[h(x_t,p)^\top (y_t - p)]
\end{align*}
is small, in this alternative definition, we fix a class of groups $\cC \subseteq {\cX \rightarrow {0,1} }$ and sum up the calibration errors over bins defined in terms of the realized forecasts $p_t$.

To simplify our presentation, we focus on the case where $\cY = [0,1]^d$ is the $d$-dimensional hypercube; however, similar constructions apply to other compact, convex sets.

We let $$V \coloneqq \{0, 1/N, \ldots, 1\}^d \subset \cY$$ be a grid of size $(N+1)^d$ and let $\{B_v\}_{v \in V}$ be the partition of $[0,1]^d$ into axis-aligned hypercubes of side length $1/N$ centered at each $v\in V$. Given this notation, we define swap ECE multicalibration:

\begin{definition}[Swap ECE Multicalibration]\label{eq:swap_ece}
Consider the online protocol where at every time  $t$, Nature selects features $x_t \in \cX$, and then the Learner produces a distribution $\cD_t$ over forecasts $p_t \in \cY =[0,1]^d$. Lastly, Nature reveals the outcome $y_t \in \cY$. 

We say that an algorithm $\cA$ guarantees swap ECE multicalibration with respect to $\cC \subseteq \{\cX \rightarrow \{0,1\}\}$ and $V$, if for any $\delta > 0$, 
\begin{align*}
 \ECE(\cC)
  =
  \sum_{i=1}^d \sum_{v \in V} \sup_{c \in \cC}
  \biggl|\sum_{t=1}^T c(x_t)\,1\{p_t \in B_v\}\,(y_{t,i} - v_i)\biggr|
\end{align*}
is $o(T)$ with probability $1-\delta$ regardless of Nature's choices of $(x_t,y_t)$. Here, $p_t$ is drawn from $\cD_t$. Nature chooses $y_t$ with knowledge of $\cD_t$, but not of $p_t$.
\end{definition}

In this section, we show that by using a specific choice of no-regret algorithm within our main reduction (\Cref{algo:mc from ol}), we can also achieve this notion of multicalibration as well. In doing so, we highlight how our generalized definition allows for simpler analyses for a wide range of problems that lie downstream of multicalibration (e.g. omniprediction). 

The construction is very simple. Given any finite set $\cC$, define the finite class of test functions
\begin{align}
\label{eq:l1_class}
  \cH_{\cC}
  \;=\;
  \bigl\{(x,p) \mapsto s \cdot c(x) \cdot 1\{p \in B_v\} \cdot e_i
  \;:\; c \in \cC,\; v \in V,\; s \in \{\pm 1\},\; i \in [d]\bigr\},   
\end{align}
where $e_i$ is the $i$-th standard basis vector, so that
$h(x,p)^\top(y-p) = s \cdot c(x) \cdot 1\{p \in B_v\}(y_i - p_i)$
is a scalar. We have $|\cH_{\cC}| = 2|\cC|(N+1)^d d$.

Applying \Cref{algo:mc from ol} with the Multiscale Multiplicative Weights with Correction algorithm \citep{chen2021impossible} as the external
regret minimizer over $\cH_{\cC}$, and solving each EVI to precision
$\eps_t = 1/t^2$ so that $\EVI_T \leq \cO(1)$, we obtain from
\Cref{thm:ol to mc} that for every $h \in \cH_{\cC}$,
\[
  \MCerr_T(h) \;\leq\; \Regret_T(h) + \cO(1).
\]
Recall from \Cref{corr:noarov} (\Cref{eq:regret_mmm}) that this no-regret algorithm guarantees that 
\begin{align*}
	\Regret_T(h) = \sum_{t=1}^T f_t(h_t) - \sum_{t=1}^T f_t(h) \leq \widetilde{\cO}
	\left(\sqrt{\log(|\cH_{\cC}|)\sum_{t=1}^T f_t(h)^2}\right).
\end{align*}
By definition of the losses in \Cref{algo:mc from ol}, taking a generic test function $h$ from $\cH_{\cC}$ \eqref{eq:l1_class} indexed by $(c,v,s,i)$ we have that
\[
  f_t(h)
  \;=\;
  -\E_{p_t \sim \cD_t}\bigl[h(x_t, p_t)^\top(y_t - p_t)\bigr] = - s \E_{p_t \sim \cD_t}\bigl[c(x_t)1\{p_t \in B_v\}(y_{t,i} - p_{t,i})\bigr].
\]
Furthermore,  since $1\{p \in B_v\}^2 = 1\{p \in B_v\}$ and
$(y_{t,i} - p_{t,i})^2 \leq 1$, Jensen's inequality gives
\begin{align*}
    f_t(h)^2 \leq \E_{p_t \sim \cD_t}[1\{p_t \in B_v\}] =  \Pr_{p_t \sim \cD_t}[p_t \in B_v] \eqqcolon q_{v,t}.
\end{align*}
Putting all of this together, we get that,
\[
  \MCerr_T(h)
  \;\leq\;
  \widetilde{\cO}\!\left(\sqrt{\log|\cH_{\cC}| \cdot \sum_{t=1}^T q_{v,t}}\right)
  \;=\;
  \widetilde{\cO}\!\left(\sqrt{\log|\cH_{\cC}| \cdot N_v}\right),
\]
where $N_v \coloneqq \sum_t q_{v,t}$ is the expected number of rounds in which the forecast lands in bin $B_v$. Since for every $h \in \cH_{\cC}$, the function $-h$ is also in $\cH_{\cC}$, we can take absolute values and conclude that for every $c \in \cC$, and  $v \in V$:
$i \in [d]$:
\begin{align}
\label{eq:ece_mc bound}
  \biggl|\sum_{t=1}^T \E_{p_t \sim \cD_t}
  \bigl[c(x_t)\,1\{p_t \in B_v\}\,(y_{t,i} - p_{t,i})\bigr]\biggr|
  \leq\
  \widetilde{\cO}\!\left(\sqrt{\log|\cH_{\cC}| \cdot N_v}\right).
\end{align}
Next, we apply a Martingale argument to move from a bound on the expectation to a high-probability statement over the realized value of $p_t$. For a fixed, generic $h\in \cH_{\cC},$ define
\begin{align}
\label{eq:ece_martingale_def}
Z_t \coloneqq \E_{p_t \sim \cD_t}
  \bigl[c(x_t)\,1\{p_t \in B_v\}\,(y_{t,i} - p_{t,i})\bigr] - 
  \bigl[c(x_t)\,1\{p_t \in B_v\}\,(y_{t,i} - p_{t,i})\bigr]
\end{align}
Because $y_t$ and $p_t$ are drawn independently conditional on the history up to time $t$, $\E[Z_t|\cF_{t-1}] = 0$ where $\cF_{t-1} = \{x_j,y_j,\cD_j,p_j\}_{j=1}^{t-1} \cup  \{x_t,\cD_t,y_t\}$. Furthermore, $|Z_t| \leq 1$ almost surely and 
\begin{align*}
    \E[Z_t^2 | \cF_{t-1}] &= \mathrm{Var}\left[ c(x_t)\,1\{p_t \in B_v\}\,(y_{t,i} - p_{t,i}) \mid \cF_{t-1}\right] \\ 
    & \leq \E\left[ c(x_t)^2\,1\{p_t \in B_v\}^2\,(y_{t,i} - p_{t,i})^2 \mid \cF_{t-1}\right]  \\ 
    & \leq \Pr_{p_t \sim \cD_t}[p_t \in B_v] = q_{v,t}
\end{align*}
Bernstein's inequality guarantees that for $N_v = \sum_{t=1}^T q_{v,t}$, 
\begin{align*}
    \Pr\left[ |\sum_{t=1}^T Z_t|  \geq u \right] \leq 2\exp\left( \frac{-u^2}{2N_v + 2u/3}\right).
\end{align*}
Inverting the bound, we conclude that the following inequality holds with probability $1-\delta$ over the samples drawn from $\cD_t$ for any fixed $h\in \cH_{\cC}$
\begin{align}
\label{eq:high_prob_ECE}
    |\sum_{t=1}^T Z_t| \leq \sqrt{2N_v\log(2/\delta)} + \log(2/\delta)
\end{align}
Now, applying a union bound over all $h \in \cH_{\cC}$ and combining \Cref{eq:ece_martingale_def,eq:high_prob_ECE,eq:ece_mc bound} using the triangle inequality, 
we get that the following holds with high probability over every $h\in \cH_{\cC}$:
\[
  \biggl|\sum_{t=1}^T
  \bigl[c(x_t)\,1\{p_t \in B_v\}\,(y_{t,i} - p_{t,i})\bigr]\biggr|
  \leq\
  \widetilde{\cO}\!\left(\sqrt{N_v \cdot \log(|\cH_{\cC}| / \delta)} + \log(|\cH_{\cC}|/\delta)\right).
\]
The final step is to move from $y_{t,i} - p_{t,i}$ to $y_{t,i} - v_i$. 
Notice that for $p_t \in B_v$, we have $|p_{t,i} - v_i| \leq 1/(2N)$ in each coordinate,
so by the triangle inequality the following holds for every $c\in \cC$
\begin{align*}    
  \biggl|\sum_{t=1}^T 
  \bigl[c(x_t)\,1\{p_t \in B_v\}\,(y_{t,i} - v_i)\bigr]\biggr|
  &\leq
  \biggl|\sum_{t=1}^T 
  \bigl[c(x_t)\,1\{p_t \in B_v\}\,(y_{t,i} - p_{t,i})\bigr]\biggr|
  + \frac{\sum_{t=1}^T 1\{p_t \in B_v\}}{2N}. 
\end{align*}
Summing over all $d$ coordinates, all bins $v \in V$, and using $\sum_v N_v = T$, since the bins partition $[0,1]^d$, we get that with probability $1-\delta$
\begin{align*}
  \sum_{i=1}^d \sum_{v \in V}
   \sup_{c \in \cC}\biggl|\sum_t \bigl[c(x_t)\,1\{p_t \in B_v\}\,(y_{t,i} - v_i)\bigr]\biggr|
  &\leq\;
  d\sum_{v \in V} \widetilde{\cO}\!\left(\sqrt{\log|\cH_{\cC}| \cdot N_v}\right)
  + d|V|\log(|\cH_{\cC}/\delta|)+ \frac{dT}{2N} \\
  &\leq\;
  \widetilde{\cO}\!\left(d\sqrt{(N+1)^d \cdot T \cdot d\log(d\cdot |\cC|\cdot N)}\right)
  + \frac{dT}{2N}.
\end{align*}
In the last step we used Cauchy-Schwarz,
\begin{align*}
\sum_{v \in V} \sqrt{N_v} = \sum_{v \in V} \sqrt{N_v} \cdot 1  \leq \sqrt{|V|} \cdot \sqrt{\sum_v N_v} = \sqrt{(N+1)^dT},
\end{align*}
and substituted $\log|\cH_{\cC}| = \log(2|\cC|(N+1)^d d) = \cO(d \log(|\cC| d N))$.
\[
  \ECE(\cC)
  \;\leq\;
  \widetilde{\cO}\!\left(d\sqrt{(N+1)^d \cdot T \cdot d} \right)
  + \frac{dT}{2N}.
\]
Lastly, we balance between both terms by setting $N \asymp  (T/d)^{1/(d+2)}$ which yields the final $\widetilde{\cO}(T^{(d+1)/(d+2)})$ bound. We summarize this argument in the following corollary:
\begin{corollary}
\label{corr:ece}
Let $\cC \subseteq \{\cX \to \{0,1\}\}$ be a finite class. The forecaster produced by \Cref{algo:mc from ol} with the Multiscale Multiplicative Weights with Correction algorithm \cite{chen2021impossible} as the regret minimizer over $\cH_{\cC}$ defined  as in \Cref{eq:l1_class} with $N \asymp  (T/d)^{1/(d+2)}$ guarantees that with probability $1-\delta$
\[
  \ECE(\cC)
  \;\leq\;
  \widetilde{\cO}\!\left(d^{(d+3)/(d+2)} \cdot T^{(d+1)/(d+2)} \cdot \sqrt{\log|\cC|}\right).
\]  
\end{corollary}
In 1 dimension, the bound becomes $\widetilde{\cO}(T^{2/3})$ which is the optimal rate for this problem as recently shown by \citet*{collina2026optimal}. The $T^{(d+1)/(d+2)}$ rate is in fact optimal for general $d>1$ \cite{personalcomm}. 

\begin{remark}\label{rem:high dim}
We highlight how this stands in contrast to the definition of multicalibration (\Cref{eq:omc}) we study for which $\sqrt{T}$ rates are achievable even in the high-dimensional setting. We hope that this ECE result illustrates to the reader why we focus on the version of multicalibration from \cref{eq:omc}. While it can be used, as per \Cref{corr:ece}, to guarantee more stringent notions of calibration such as this ECE result, it is already strong enough to imply other desiderata like omniprediction (see e.g. the previous subsection) and $\Phi$-regret minimization with sharp $\sqrt{T}$ rates, rates that are not achievable via arguments that depend on ECE versions of calibration.

Lastly, we remark that this construction provides a unifying analysis of previous algorithms  for ECE calibration of a scalar outcome ($\cY = [0,1]$) such as \cite{hu2025efficient}. While these run separate instantiations of online algorithms per bin, our simplified construction built on EVIs uses a single online learning routine. 
\end{remark}

\section{Algorithmic Implications for $\Phi$-Regret Minimization}
\label{sec:applications Phi}

We now expand on how our new path between external regret minimization and $\Phi$-regret minimization, via  multicalibration and EVIs, enables algorithms with improved rates and simpler analyses. 

Recent approaches to $\Phi$-regret minimization have considered ever-increasing sets of transformations $\Phi$ and general convex action sets $\cZ$ \citep{daskalakis2025efficient,arunachaleswaran2025swap,zhang2025learning}.
In all these applications, a recurring technical challenge is that GGM requires that the transformations $\phi \in \Phi$ have fixed points in $\cZ$. And unfortunately, the set of functions that admit fixed points is not convex. 

To circumvent the issue, \citet{daskalakis2025efficient} was the first to introduce intricate \emph{semiseparation} techniques to operate on supersets of linear endomorphisms of the action set $\cZ$ that admit fixed points in $\cZ$. Extensions of these techniques have been critical to enable the results of \citet{arunachaleswaran2025swap} and \citet{zhang2025learning}. Given the sophistication of semi-separation, it should not come as a surprise that these more recent results come at the cost of extremely technically involved algorithms and undesirably high dependency on the dimensionality of the action set. A recent paper of \citet*{arxiv26} ameliorated the issue by showing that, depending on the goal, semiseparation can be replaced by response-based approachability \citep{bernstein2013response}, although the method is not known to be efficient for low-degree polynomial deviations, and the regret bounds require a (possibly non-polynomial-time) preprocessing of the strategy set $\cZ$.

In this section, we show that the EVI-based reduction from online learning to multicalibration (\cref{thm:ol to mc}), combined with the reduction from multicalibration to $\Phi$-regret minimization (\cref{thm:mc to phi}), yields an EVI-based, forecasting-centric alternative to the otherwise fixed-point-based construction of \citet{gordon2008no} that overcomes many of the aforementioned shortcomings.

\subsection{Improving Rates for Linear Swap Minimization}\label{sec:lin swap}

To fix ideas, consider the main problem studied by \citet*{daskalakis2025efficient}: constructing a linear swap-regret minimizer for a well-bounded (\cref{def:wr}) convex compact set $\cZ \subseteq \R^d$, facing losses from a well-bounded convex compact set $\cL \subseteq \R^d$. For both sets, we assume efficient oracle access.
In this setting, the set of transformations $\Phi$ is the set of all linear mappings from $\cZ$ to $\cZ$, \emph{i.e.}, all linear endomorphisms of $\cZ$. 
We assume that the set of loss vectors $\cL$ has maximum Euclidean norm $L$, each point in $\cZ$ has maximum Euclidean norm $B$, and that the geometry of $\cZ$ is such that the every linear endomorphism of $\cZ$ has maximum spectral norm $S$. We remark that the Fritz John's theorem implies that any convex body admits a linear transformation ensuring that $S = B = d$.

In this setting, the classical GGM approach to $\Phi$-regret minimization would reduce to regret minimization over $\Phi$. Unfortunately, the latter computational task is daunting, since testing membership in the set of endomorphisms is provably hard \citep{daskalakis2025efficient}.
By instead taking a forecasting point of view, we can construct a no-linear-swap regret minimizer for $\Phi$ starting from any online forecaster for the convex loss set $\cL \subseteq \R^d$ that is multicalibrated with respect to the class of tests
\[
    \cH_\Phi \coloneqq \{ (x, p) \mapsto  \sigma(p) - \phi(\sigma(p)) : \phi \in \Phi\},
\]
where $\sigma$ is the best-response function on $\cZ$ (see \Cref{eq:H_PHI_DEF}). Since the $\phi \in \Phi$ are linear functions, each element of $\cH_\Phi$ is a linear function of $\sigma(p)$. Hence, one way to multicalibrate with respect to $\cH_\Phi$ is to multicalibrate with respect to the ball of \emph{all} linear functions
\[
    \cH \coloneqq \{ (x, p) \mapsto M \sigma(p) : \|M\|_F \le \rho\},
\]
where we set the parameter $\rho > 0$ to be large enough. In particular, since we assumed that each linear endomorphism has maximum spectral norm $S$, and $\|\cdot\|_F \le \sqrt{d}\|\cdot\|_\text{spectral}$, it suffices to set $\rho = \sqrt{d} + \sqrt{d}S$ to ensure capturing all functions in $\cH_\Phi$.

By leveraging \cref{thm:ol to mc,thm:mc to phi}, we can then reduce $\cH$-multicalibration to online learning on $\cH$ with linear losses where we optimize over matrices $\{M \in \R^{d\times d}: \|M\|_F\le \rho\}$. Crucially, several standard no-regret algorithms can be applied off-the-shelf on the Frobenius ball. In particular, by invoking the standard analysis of online projected gradient descent, we obtain the following.

\begin{proposition}\label{prop:linear}
    Let $\cZ,\cL \subseteq \R^d$ be well-bounded convex bodies given via efficient oracle access. We let $L$ denote the maximum Euclidean norm of any vector in $\cL$, $B$ denote the maximum Euclidean norm of any point in $\cZ$, and $S$ denote the maximum spectral (operator) norm of any linear endomorphism of $\cZ$.
    There exists a linear swap minimization algorithm for strategy set $\cZ$ and loss set $\cL$ with polynomial-time iterations that guarantees $\cO(BLS\sqrt{d T})$ linear swap regret for any time horizon $T$, no matter the sequence of losses $\ell_1, \dots, \ell_T \in \cL$.
\end{proposition}
\begin{proof}
    We construct the $\cH$-multicalibrated forecaster by instantiating \cref{thm:ol to mc} using projected gradient descent on $\{M \in \R^{d\times d} : \|M\|_F \le \rho\}$. At every time $t$, the algorithm outputs a matrix $M_t$ inducing the test $p \mapsto M\sigma(p) \in \cH$, and incurs the linear loss
    \begin{align*}
        f_t(M) = -\E_{p_t\sim\cD_t}\Big[(M \sigma(p_t))^\top (\ell_t - p_t)\Big] = -\Big\langle M, \E_{p_t \sim\cD_t}\Big[(\ell_t - p_t)\sigma(p_t)^\top\Big]\Big\rangle_F.
    \end{align*}
    Since $p_t,\ell_t\in\cL$ and $\sigma(p_t)\in \cZ$, the norm of the gradient of the loss function is therefore bounded by $\cO(LB)$. Furthermore, by construction of $\cH$ the Frobenius norm of the domain is bounded above by $\sqrt{d} + \sqrt{d}S$. Since $S \ge 1$ (the identity map is always an endomorphism of $\cZ$, and has unit spectral norm), the bound on the norm of the domain is therefore $\cO(\sqrt{d}S)$. Since the regret incurred by projected gradient descent scales as the norm of the domain, times the norm of the loss set, times $\sqrt{T}$, we obtain a regret bound on $\cH$ of the order $\cO(BLS\sqrt{d T})$. By \cref{thm:ol to mc}, if we solve the EVI problems to precision $\eps_t = t^{-2}$, we get that $\EVI_T \leq \cO(1)$ and for any linear endomorphism $\phi$:
    \begin{align*}
        \Regret^T_\Phi(\phi) \leq \MCerr_{T}(h_\phi) \leq \EVI_T + \Regret_T(h_\phi) \leq \cO(BLS\sqrt{dT}).
    \end{align*}
    We remark that the construction leads to polynomial-time iterates, in light of \cref{fact:EVI} and the polynomial dimension of $M$.
\end{proof}

If the set $\cZ$ is already in John's position and under the standard normalization assumption that $|\ell^\top z|\le1$ for all $\ell\in\cL, z\in\cZ$ (i.e., $B=d, L=1, S=d$), this yields a rate of $\cO(d^{5/2}\sqrt{T})$ which improves upon the rate $\cO(d^4 \sqrt T)$ of \citet{daskalakis2025efficient}. 

However, the benefit of this forecasting approach to $\Phi$-regret is not only its improved dependence on the dimension $d$, but rather that it establishes a dramatically simpler algorithmic template.  
Indeed, while GGM-based approaches necessarily need to grapple with the complicated geometry of endomorphisms (expected fixed points of non-endomorphic maps need not exist, invalidating the GGM reduction), the multicalibration route is much more permissive: it is always easy to multicalibrate with respect to more tests. In our case, we promoted the set of tests from $\cH_\Phi$, to the set of \emph{all} linear multicalibration tests of sufficiently large norm, and the latter is no harder than performing online projected gradient descent on a Euclidean ball. We did not have to concern ourselves with whether the transformations in the enlarged set $\cH \supseteq \cH_\Phi$ admit fixed points in $\cZ$.

\subsection{No-$\Phi$-Regret with Respect to RKHS Deviations}

Follow-up work to \cite{daskalakis2025efficient} has begun to explore $\Phi$-regret minimization beyond linear deviations, including cases where the deviations are low-degree polynomials or endomorphisms given by a linear combination of a finite, polynomial number of nonlinear features \citep{zhang2025learning}.

In this section, we point out that the same approach outlined in \cref{sec:lin swap} directly extends to these settings too. In fact, by multicalibrating over infinite sets of test functions using the techniques developed in \cref{sec:defensive}, we can obtain, with no effort and for the first time, $\Phi$-regret minimization algorithms in which $\Phi$ contains all transformations that lie in an RKHS ball (including RKHS with potentially infinitely many features, such as the Gaussian RKHS). This recovers low-degree polynomials as a special case.

\begin{theorem}
    Let $\cZ$ be convex and compact with efficient oracle access, and let $\Phi$ be the subset of endomorphic functions $\phi : \cX \times \cZ \to \cZ$ belonging to a vector-valued RKHS $\cH \subseteq \{\cX \times \cZ \to \R^d\}$ with corresponding matrix-valued kernel $\Gamma((x,z),(x',z'))$. Assume $\Gamma$ can be evaluated in polynomial time for any inputs and let $\rho \coloneqq  1+ \sup_{\phi\in\Phi} \|\phi\|_\cH$ scale additively with the maximum RKHS norm of any of the functions in $\Phi$. Then, using Defensive Forecasting (\Cref{alg:defensive_forecasting}, \Cref{prop:defensive}) with the following kernel,
    \begin{align}        
    \label{eq:rkhs_kernel}
        \Gamma'((x, p), (x', p')) = \sigma(p) \sigma(p')^\top  + \Gamma((x,\sigma(p)), (x', \sigma(p'))),
    \end{align}
    as the loss forecaster in the black-box reduction between multicalibration and $\Phi$-regret minimization (\Cref{algo:phi from mc}) produces distributions $\mu_t$ over $\cZ$ satisfying the following for any sequence $\{\ell_t\}_{t=1}^T$: 
    \begin{align*}
                \sup_{\phi \in \Phi} \sum_{t=1}^T \E_{z_t \sim \mu_t} \Big[\ell_t^\top(z_t - \phi(x_t,z_t))\Big]  \le \cO\left(\rho \sqrt{\sum_{t=1}^T\E_{p_t \sim \cD_t}(p_t - \ell_t)^\top \Gamma'((x_t,p_t),(x_t,p_t))(p_t - \ell_t)}\right).
    \end{align*}
\end{theorem} 
\begin{proof}
    Since $\cZ$ has an efficient linear optimization oracle, evaluations of the best response map $\sigma : p \mapsto \argmin_{z \in \cZ} p^\top z$ are also efficient.
    From \cref{thm:mc to phi}, it suffices to construct an $\cH_\Phi$ multicalibrated forecaster for the loss set $\cL = \{\ell \in \R^d : |\ell^\top z| \le 1 ~~\forall z \in \cZ\}$, where
    \begin{align}
    \label{eq:RKHS_test_functions}
            \cH_\Phi = \{(x, p) \mapsto \sigma(p) - \phi(\sigma(p)) : \phi \in \Phi\}.
    \end{align}
    The crucial observation is that since every deviation $\phi$ belongs to an RKHS with kernel $\Gamma((x, p), (x', p'))$, then $\phi(\sigma(p))$ belongs to an RKHS with kernel $\Gamma((x, \sigma(p)), (x',\sigma(p')))$ and the functions $\sigma(p) - \phi(x,\sigma(p))$ belong the RKHS corresponding to the kernel in \Cref{eq:rkhs_kernel}. 
    
    This makes use of the fact that if $f_1$ is in the RKHS defined by kernel $\Gamma_1$ and $f_2$ is in the RKHS defined by $\Gamma_2$, then $\Gamma_1 + \Gamma_2$ is a kernel that contains $f_1 + f_2$. Furthermore, the norm of $f_1 + f_2$ in the new RKHS is at most the square root of the sum of the squared norms.
    
    The functions $\sigma(p)$ lie in the RKHS corresponding to the linear kernel $\sigma(p) \sigma(p')^\top$ and have norm at most 1. Furthermore, $\phi(x,\sigma(p))$ lie in the RKHS corresponding to $\Gamma((x,\sigma(p)), (x', \sigma(p')))$ and have norm at most $\rho$ by our assumption on $\Gamma$ and Theorem 5.14 in \cite{paulsen2016introduction}. Therefore, \Cref{prop:defensive} guarantees that the multicalibration error of the forecaster with respect to the functions in \Cref{eq:RKHS_test_functions} is at most
    \begin{align*}
        \cO\left(\rho \sqrt{\sum_{t=1}^T\E_{p_t \sim \cD_t}(p_t - \ell_t)^\top \Gamma'((x_t,p_t),(x_t,p_t))(p_t - \ell_t)}\right).
    \end{align*}
    The bound on $\Phi$-regret follows from our main result in \Cref{thm:mc to phi}.
\end{proof}

Note that this result is in fact stated for the stronger notion of \textit{contextual} regret where the decision maker can observe features each round, and considers deviations of their actions that depend on such contexts. Furthermore, we state the result for general kernels $\Gamma$ to illustrate a tighter sequence-dependent bound. However, virtually all known kernels (polynomial, Gaussian, etc.) have $\|\Gamma((x,p),(x,p))\|_{\mathrm{op}}$ bounded by a constant $B$. In this case, we can replace the bound by one of the form $\cO(\sqrt{T})$ where the $\cO(\cdot)$ hides dependence on $B$ and the norms of $\ell_t$.

\section*{Acknowledgements}

We would like to thank Haipeng Luo, Aaron Roth, and Vatsal Sharan for helpful comments and discussion. GF was supported in part by the National Science Foundation award CCF-2443068, the Office of Naval Research grant N000142512296, and an AI2050 Early Career Fellowship.

\printbibliography

\end{document}